\documentclass[11pt,a4paper]{article}
\usepackage{times}
\usepackage{graphicx}
\usepackage{latexsym} 
\usepackage{csquotes}
\usepackage{stmaryrd}
\usepackage{amssymb}
\usepackage{paralist}
 \usepackage{thm-restate} 
\usepackage{thmtools}
\usepackage{amsmath}
\usepackage[amsmath]{ntheorem}
\usepackage{cite}

\newcommand{\AF}{F} 
\newcommand{\SETAF}{SF}
\newcommand{\ADF}{DF} 
\newcommand{\EAFC}{EFC} 
\newcommand{\AFN}{FN}

\newcommand{\ext}{\mathit{X}}

\newcommand{\charF}{\mathcal{F}_{\AF}}

\newcommand{\charEAFC}{\mathcal{F}_{\EAFC}}

\newcommand{\tvt}{\mathbf{t}}
\newcommand{\tvf}{\mathbf{f}}

\theoremstyle{plain} 
\newtheorem{theorem}{Theorem}[section]
\newtheorem{lemma}[theorem]{Lemma}
\newtheorem{proposition}[theorem]{Proposition}

\newtheorem{translation}{Translation}{\bfseries}{\itshape} 

\newtheorem{definition}{Definition}
\theoremstyle{definition}
\newtheorem{example}{Example}  

\newenvironment{proof}%
   {\begin{trivlist}\item[]\textbf{Proof.}}%
   {\hfill$\Box$\end{trivlist}}

\begin{document}

\title{Understanding the Abstract Dialectical Framework \\ Preliminary Report}

\author{Sylwia Polberg\thanks{The author is a member of the Vienna PhD School of Informatics. This research was funded by project I1102 supported by the 
Austrian Science Fund FWF. The author is currently
supported by EPSRC Project EP/N008294/1 \enquote{Framework for Computational Persuasion}}}
 \date{University College London,\\ Gower Street 66--72, London WC1E 6EA, United Kingdom}

\maketitle 

\begin{abstract} 
Among the most general structures extending the framework by Dung 
are the abstract dialectical frameworks (ADFs). They come equipped with various types of semantics, with the most prominent --
the labeling--based one -- analyzed in the context of computational complexity, signatures, instantiations and software support. 
This makes the abstract dialectical frameworks valuable tools for argumentation.
However, there are fewer results available concerning the relation between the ADFs and other argumentation frameworks. 
In this paper we would like to address this issue by introducing a number of translations from various formalisms into ADFs. 
The results of our study show the similarities
and differences between them, thus promoting the use and understanding of ADFs.
Moreover, our analysis also proves their capability to model many of the existing
frameworks, including those that go beyond the attack relation. 
Finally, translations allow other structures to benefit from the research on ADFs in general and 
from the existing software in particular.
\end{abstract} 

\section{Introduction}

Argumentation has become an influential subfield of AI
\cite{incoll:arglegal,incoll:argdia1,incoll:argdia2,DBLP:journals/jbi/FoxGPABSRV10,DBLP:journals/artmed/HunterW12,DBLP:journals/ail/AtkinsonBM06}.
Within this domain, we distinguish the abstract argumentation, at the heart of which
lies Dung's framework (AF) \cite{article:dung}. A number of its generalizations has been proposed \cite{general}, including
the abstract dialectical framework (ADF)\cite{inproc:adf}. ADFs come equipped with various types of semantics 
\cite{report:strass,tofix:newadf,inproc:adm,polberg-stairs14,report:semanticsrev}, the most prominent of which --
the labeling--based one -- analyzed in the context of computational complexity \cite{strass-wallner14complexity}, signatures \cite{Puehrer2015}
instantiations \cite{strass13instantiating} and software support \cite{Ellmauthaler2014a}. 
This makes ADFs valuable tools for argumentation. Unfortunately, their unusual structure
can be a deterrent against their more widespread use. Moreover, at the first glance it is also difficult to say what is the relation between
the ADFs and the other argumentation frameworks, in particular those that can express support \cite{incoll:bipolar,article:newbaf,incoll:newafn,inproc:easafn}.

In this paper we would like to tackle these issues by introducing a number of translations from various formalisms into the ADFs. This includes
the Dung's framework \cite{article:dung}, the Nielsen's and Parson's framework with joint attacks \cite{incoll:setatt}, 
the extended argumentation framework \cite{ModgilP10} and
the argumentation framework with necessities \cite{incoll:newafn}.  
The results of our study show the similarities
and differences between ADFs and other argumentation formalisms, thus promoting the use and understanding of ADFs.
Moreover, our analysis also proves their capability to model many of the existing
frameworks, including those that go beyond the attack relation. Furthermore, a wider range of extended
argumentation frameworks can be translated into ADFs than into AFs \cite{Modgil11}.

This paper is structured as follows. In Sec. \ref{sec:frams} and \ref{sec:adf} we recall the aforementioned argumentation frameworks. 
We also provide a discussion on certain design differences between the ADFs and the other structures. In Sec. \ref{sec:trans}
we present our translations. We close the paper with final remarks and comments on shifting other frameworks to ADFs.

\section{Argumentation Frameworks} 
\label{sec:frams}

In this section we will recall the relevant argumentation frameworks and their extension--based semantics.
Despite the various structural differences between the frameworks, their semantics tend to follow the design patterns 
established by Dung \cite{article:dung}. We can obtain most of our them by combining conflict--freeness, acceptability and various
ways to maximize or minimize our extensions. Thus, many frameworks tend to redefine these \enquote{building blocks}, and then reuse
the original (or similar) definitions from \cite{article:dung}. Therefore, when recalling the relevant structures in this section, we will mostly
provide the necessary notions and reuse the existing formulations. Finally, please note we will be focusing on finite structures.

\subsection{Dung's Argumentation Framework}

Let us start with the famous Dung's framework \cite{article:dung},
which is based on binary attack. 

\begin{definition}
A \textbf{Dung's abstract argumentation framework} (AF) is a pair $\AF=(A, R)$, where $A$ is a set of \textbf{arguments} and
$R \subseteq A \times A$ is the \textbf{attack} relation.
\end{definition}
\begin{definition}
\label{def:baseaf}
Let $\AF = (A, R)$ be a Dung's framework and $\ext \subseteq A$ a set of arguments. 
\begin{itemize}
\item the \textbf{attacker set} of $\ext$ is $\ext^- = \{ a \mid \exists b \in \ext, a R b\}$
\item the \textbf{discarded set} of $\ext$ is $\ext^+ = \{a \mid \exists b \in \ext, b R a\}$.
\item $\ext$ \textbf{defends}\footnote{Defense is often
substituted with acceptability, i.e. $\ext$ defends $a$ iff $a$ is acceptable w.r.t. $\ext$.} 
an argument $a \in A$ iff every argument $b \in A$ that attacks $a$ is in $\ext^+$.
\item $\ext$ is \textbf{conflict--free} in $\AF$ iff there are no $a,b \in \ext$ s.t. $a$ attacks $b$. 
\end{itemize}
\end{definition}

\begin{definition}\label{def:semantics}
Let $\AF = (A, R)$ be an AF.
A set $\ext \subseteq A$ is:
\begin{itemize}  
\item \textbf{admissible} in $\AF$ iff it is conflict--free in $\AF$ and defends in $\AF$ all of its members. 
\item \textbf{preferred} in $\AF$ iff it is maximal w.r.t. $\subseteq$ admissible in $\AF$. 
\item \textbf{complete} in $\AF$ iff it is admissible and every $a \in A$ that is defended by $\ext$, is in $\ext$.
\item \textbf{grounded} in $\AF$ iff it is the least fixed point of the characteristic operator $\charF : 2^A \rightarrow 2^A$ defined as 
$\charF(\ext) = \{a \mid a $ is defended by $ \ext $ in $ \AF\}$
\item \textbf{stable} in $\AF$ iff it is conflict--free in $\AF$ and $A \setminus \ext = \ext^+$.
\end{itemize}
\end{definition} 

The extensions between different semantics can be related to each other in a number of ways \cite{article:dung}, however,
it is usually the following properties that will hold:

\begin{theorem}
\label{thm:dung2}
Let $\AF = (A,R)$ be an AF.
The following holds:
\begin{enumerate}  
\item Every stable extension of $\AF$ is also preferred, but not vice versa. 
\item Every preferred extension of $\AF$ is also complete, but not vice versa.
\item The grounded extension of $\AF$ is the least w.r.t. $\subseteq$ complete extension of $\AF$. 
\end{enumerate}
\end{theorem}

\subsection{Framework with Sets of Attacking Arguments}

In some cases, a single argument might not be enough 
to carry out an attack on another argument. 
For example, all of the means, motive, opportunity and evidence might be required to prove guilt.
In order to grasp such problems, 
a framework with group conflict was developed \cite{incoll:setatt}. The semantics of SETAFs are almost identical to the AF ones. 
Given a set $\ext \subseteq A$, the attacks will now be
carried out not by single arguments in $\ext$, but its subsets. Thus, in the interest of space, we will not formally give their definitions. 

\begin{definition}
A \textbf{framework with 
sets of attacking arguments} (SETAF) is a pair $\SETAF = (A, R)$,  where $A$ is the set of 
\textbf{arguments} and $R\subseteq (2^A \setminus \emptyset)\times A$ is the \textbf{attack} relation.
\end{definition} 
\begin{example}\label{exa:setaf}
Let us consider the SETAF $\SETAF = (A,R)$, where $A = \{a,b,c,d,e\}$ and 
$R= \{(\{a\},c)$, $(\{b\},a)$,$(\{b\},b)$, $(\{c\},d)$, $(\{e\},a)$, $(\{b,d\},e)\}$. 
The only admissible extensions are $\emptyset$ and $\{c,e\}$; both of them are complete. 
$\{c,e\}$ is the preferred extension, while $\emptyset$ is grounded. Because of $b$,
this particular framework has no stable extensions.
\end{example}
 
\subsection{Extended Argumentation Framework with Collective Attacks}

The extended argumentation framework with collective defense attacks \cite{ModgilP10} is an improvement
of the framework studied in \cite{Modgil11,Modgil07agent,Bench-Capon2009,dagstuhlModgilL09,DunneMB10,Modgil14,implemeaf}.
It introduces the notion of defense attacks, which occur between sets of arguments and binary conflicts. 
They can \enquote{override} a given attack due to e.g. the target's importance, which is a common approach
in the preference--based argumentation \cite{AmgoudC02,Kaci2008,incoll:prefunified,report:vafs,article:newpafs}.  
The added value of defense attacks is the fact that the arguments carrying them out can also be attacked and
questioned.

\begin{definition}
\label{def:bheafc} 
An \textbf{extended argumentation framework with collective defense attacks} (EAFC) is a tuple $\EAFC = (A, R, D)$,
where $A$ is a set of \textbf{arguments}, $R \subseteq A \times A$ is a set of \textbf{attacks} and 
$D \subseteq (2^A\setminus \emptyset) \times R)$ is the set
of \textbf{collective defense attacks}. 
\end{definition}

We can observe that a given attack can be successful (referred to as a defeat) or not, depending on the presence of suitable defense attacks. 
The defense has to include not just defending the arguments, but
also a form of \enquote{protection} of the important defeats: 

\begin{definition}
Let $\EAFC = (A,R,D)$ be an EAFC and $\ext \subseteq A$ a set of arguments. 
\begin{itemize}
\item an argument $a$ \textbf{defeats$_\ext$} an argument $b$ in $\EAFC$ w.r.t. $\ext$ iff $(a,b) \in R$ and there is 
no $C \subseteq A$ s.t. $(C,(a,b)) \in D$.
\item a set of pairs $R_\ext = \{(x_1, y_1),...,(x_n, y_n) \}$ s.t. $x_i$ defeats$_\ext$ $y_i$ in $\EAFC$ and for $i=1...n$, $x_i \in \ext$, is a 
\textbf{reinstatement
set} on $\ext$ for a defeat$_\ext$ by argument $a$ on argument $b$ iff $(a,b) \in R_\ext$ and for every pair $(x,y) \in R_\ext$ and set of arguments
$C \subseteq A$ s.t. $(C,(x,y)) \in D$, there is a pair $(x', y') \in R_\ext$ for some $y' \in C$. 
\item the \textbf{discarded set} of $\ext$ is $\ext^+ = \{ a \mid \exists b \in \ext$ s.t. $b$ defeats$_\ext$ $a$ and there is a reinstatement
set on $\ext$ for this defeat$_\ext \}$.
\item $\ext$ \textbf{defends} and argument $a \in A$ in $\EAFC$ iff every argument $b \in A$ s.t. $b$ defeats$_\ext$ $a$ in $\EAFC$ is in $\ext^+$.
\item $\ext \subseteq A$ is \textbf{conflict--free} in $\EAFC$ iff there are no $a,b \in \ext$ s.t. $a$ defeats$_\ext$ $b$ in $\EAFC$.
\end{itemize}
\end{definition}

With the exception of the grounded semantics, all extensions are defined in the same way as in Def. \ref{def:semantics}.  
Unfortunately, despite these similarities, Thm. \ref{thm:dung2} cannot be entirely extended to EAFCs. 
Finally, within EAFCs we can distinguish the bounded hierarchical subclass, enforcing certain restrictions
on the attacks and defense attacks.

\begin{definition}
Let $\EAFC = (A,R,D)$ be a finitiary\footnote{An EAFC is finitiary if for every argument and attack, the collection of its (defense) attackers is finite.}
 EAFC, $\ext \subseteq A$ a set of arguments and $2^{CF}$ the set of all conflict--free sets of $\EAFC$. 
The \textbf{characteristic
function} $\charEAFC: 2^{CF} \rightarrow 2^A$ of $\EAFC$ is defined as $\charEAFC(\ext) = \{a \mid a$ is defended by $\ext$ in $\EAFC\}$. 
We define a sequence of subsets of $A$ s.t. $\charEAFC^0 = \emptyset$ and 
$\charEAFC^{i+1} = \charEAFC(\charEAFC^i)$.
The \textbf{grounded} extension of $\EAFC$ is $\bigcup_{i=0}^\infty (\charEAFC^i)$. 
\end{definition}
\begin{restatable}{theorem}{thmdungeaf}
Let $\EAFC = (A,R,D)$ be a finitary EAFC. The following holds:
\begin{enumerate}
\item Every preferred extension is complete, but not vice versa
\item Every stable extension is complete, but not vice versa
\item The grounded extension is a minimal w.r.t. $\subseteq$ complete extension
\end{enumerate}
\label{thm:dungeaf}
\end{restatable}

\begin{definition}
An EAFC $\EAFC =(A,R,D)$ is \textbf{bounded hierarchical} iff there exists a partition $\delta_H = (((A_1, R_1), D_1), ... , ((A_n, R_n), D_n))$
s.t. $D_n = \emptyset$,
$A = \bigcup_{i=1}^n A_i$, $R = \bigcup_{i=1}^n R_i$, $D = \bigcup_{i=1}^n D_i$, for every $i=1 ... n$ $(A_i, R_i)$ is a Dung's framework,
and $(c, (a,b)) \in D_i$ implies $(a,b) \in R_i$, $c \subseteq A_{i+1}$.
\end{definition}

\begin{example} {\cite{article:eaf}}
\label{ex:afraeaf} 
Let $\EAFC =(\{a,b,c,d,e,f,g\}$, $\{(a,b), (d,c), (b,e), (e,f), (f,g) \}$, $\{ (\{b\}, (d,c)), (\{c\}, (a,b)) \})$ be an EAFC. 
Let us look at some of its conflict--free extensions. We can see that $\{a,b\}$ and $\{c,d\}$ are not conflict--free.
However, both $\{a,b,c\}$ and $\{b,c,d\}$ are, due to the presence of defense attackers,.
Additionally, also $\{a,b,c,d\}$, $\{a,d,e,g\}$ and $\{b,c,a,d,f\}$ are conflict--free.
The admissible extensions of $\EAFC$ include $\emptyset$, $\{a\}$, $\{d\}$, $\{a,d\}$, $\{b,c\}$, $\{a,b,c\}$, $\{b,c,d\}$, $\{a,d,e\}$, $\{b,c,f\}$,
$\{a,b,c,f\}$, $\{b,c,d,f\}$,
$\{a,d,e,g\}$, $\{a,b,c,d\}$ and $\{a,b,c,d,f\}$. We can observe that the set $\ext = \{b,c\}$ is admissible. Neither $a$ nor $d$ defeat$_\ext$ 
any of its elements,
and thus there is nothing to defend from. The set $\{a,d,e\}$ is admissible since the defeat of $b$ by $a$ has a reinstatement
set $\{(d,c), (a,b)\}$. Although its behavior appears cyclic, it suffices for defense.
The sets $\{a,d,e,g\}$ and $\{a,b,c,d,f\}$ are complete. We can observe they are incomparable
and do not follow the typical semi--lattice structure of complete extensions. The grounded
extension is $\{a,d,e,g\}$; it is minimal, but not the least complete extension. Both $\{a,d,e,g\}$ and $\{a,d,b,c,d,f\}$
are stable and preferred. 
\end{example}

%
\subsection{Argumentation Framework with Necessities}

Various types of support have been studied in abstract argumentation \cite{incoll:bipolar,article:newbaf,incoll:newafn,inproc:easafn}. 
Due to limited space, we will focus the necessary support, though based on
the research in\cite{article:newbaf,inproc:easafn} our results can be extended to other relations as well. 
We say that a set of arguments 
$\ext$ \textit{necessary supports} $b$ if we need to assume at least one element of $\ext$ 
in order to accept $b$. Using this relation has certain important implications.
First of all, 
argument's supporters need to be present in an extension.
Secondly, an argument can be now indirectly attacked by the means of its supporters, i.e. we can \enquote{discard} an argument
not just by providing a direct conflict, but also by cutting off its support. Finally, a certain notion of a validity of an argument is introduced, 
stemming from its participation
in support cycles. It affects the acceptance and attack capabilities of an argument.
Let us now recall the framework with necessities \cite{incoll:newafn}:

\begin{definition}
An \textbf{abstract argumentation framework with necessities} (AFN) is a tuple $\AFN = (A, R, N)$ where $A$ is a set of \textbf{arguments},
 $R \subseteq A \times A$ represents the \textbf{attack} relation
and $N\subseteq (2^A \setminus \emptyset) \times A$ represents the \textbf{necessity} relation.
\end{definition} 
 
The acyclicity restrictions are defined through the powerful sequences and the related coherent sets.
By joining conflict--freeness and coherence, we obtain a new semantics which replaces conflict--freeness as the basis of stable 
and admissible extensions.
The remaining notions are defined similarly as in Def. \ref{def:semantics} and satisfy Thm. \ref{thm:dung2}.

\begin{definition}
\label{def:afnsem}
Let $\AFN = (A,R,N)$ be an AFN and $\ext \subseteq A$ a set of arguments.
An argument $a \in A$ is \textbf{powerful} in $\ext$ iff $a\in \ext$ and there is a sequence $a_0,...,a_k$ of elements of $\ext$ s.t. :
\begin{inparaenum}[\itshape i\upshape)]
\item $a_k = a$
\item there is no $B\subseteq A$ s.t. $B N a_0$
\item for $1 \leq i \leq k$: for each $B\subseteq A$s.t. $B N a_i$, it holds that $B \cap \{a_0,...,a_{i-1}\} \neq \emptyset$.
\end{inparaenum}
A set of arguments $\ext\subseteq A$ is \textbf{coherent} in $\AFN$ iff each $a \in \ext$ is powerful in $\ext$.
\end{definition} 
 
\begin{definition}
\label{def:afn1}
Let $\AFN = (A,R,N)$ be an AFN and $\ext \subseteq A$ a set of arguments.
\begin{itemize}
\item the \textbf{discarded} set of $\ext$ in $\AFN$ is defined as 
$\ext^{att} = \{ a \mid$ for every coherent $C \subseteq A$ s.t. $a \in C$, $\exists c \in C, e \in \ext$ s.t. $e R c \}$\footnote{Please
note that we do not denote the AFN discarded set with $\ext^+$ as in the previous cases in order not to confuse it with the notion
of the deactivated set from \cite{incoll:newafn}, which is less restrictive}. 
\item $\ext$ \textbf{defends} an argument $a \in A$ in $\AFN$ iff $\ext \cup \{a\}$ is coherent and for each $b \in A$, if $b R a$ then $b \in \ext^{att}$.
\item $\ext$ is \textbf{conflict--free} in $\AFN$ iff there are no $a,b \in \ext$ s.t. $a$ attacks $b$.
\end{itemize}
\end{definition}

\begin{definition}
\label{def:afn2}
Let $\AFN = (A,R,N)$ be an AFN.
A set of arguments $\ext \subseteq A$ is:
\begin{itemize}
\item \textbf{strongly coherent} in $\AFN$ iff it is conflict--free and coherent in $\AFN$
\item \textbf{admissible} in $\AFN$ iff it is strongly coherent and defends all of its arguments in $\AFN$. 
\item \textbf{stable} in $\AFN$ iff it is strongly coherent in $\AFN$ and $\ext^{att} = A\setminus \ext$.
\end{itemize}
\end{definition}
 
\begin{example}
\label{ex:afn}
Let $(\{a,b,c,d,e,f\}, \{(a,e), (d,b), (e,c), (f,d)\}, \{(\{b,c\},a), (\{f\},f)\})$ \\
be an AFN.
Its coherent sets include $\emptyset$, $\{a,b\}$, $\{a,c\}$, $\{b\}$, $\{c\}$, $\{d\}$, $\{e\}$ 
and any of their combinations.
In total, we have six admissible extensions.
$\emptyset$ is trivially admissible. So is $\{d\}$ due to the fact that $f$ does not possess a powerful sequence in $\AFN$. 
However, $\{e\}$ is not admissible; it does not attack one of the coherent sets of $a$, namely $\{a,b\}$. Fortunately, $\{d,e\}$ is already
admissible. We can observe that $b$ can never be defended and will not appear in an admissible set. The two final extensions
are $\{a,c\}$ and $\{a,c,d\}$.
The sets $\{d\}$, $\{d,e\}$ and $\{a,c,d\}$ are our complete extensions, with the first one being grounded and the latter two preferred.
In this case, both $\{d,e\}$ and $\{a,c,d\}$ are stable.
\end{example}

%
\section{Abstract Dialectical Frameworks}
\label{sec:adf}

Abstract dialectical frameworks have been defined in \cite{inproc:adf} and further studied in
 \cite{report:strass,tofix:newadf,inproc:adm,polberg-stairs14,report:semanticsrev,strass-wallner14complexity,Puehrer2015,strass13instantiating,Strass15}.  
Their main goal is to be able to express arbitrary relations and avoid the need of introducing a new relation set each time it is needed. 
This is achieved by the means of acceptance conditions, 
which define when an argument can be accepted or rejected. 
They can be defined either as total functions over the parents of an argument \cite{inproc:adf} or as propositional formulas over them \cite{thesis:stefan}. 

\begin{definition}
\label{def:funcadf}
An \textbf{abstract dialectical framework} (ADF) is a tuple $\ADF = (A, L, C)$, where
$A$ is a set of \textbf{arguments},
$L \subseteq A \times A$ is a set of \textbf{links} and 
$C = \{C_ a \}_{a\in A}$ is a set of \textbf{acceptance conditions}, one condition per each argument. 
An acceptance condition is a total function $C_a : 2^{par(a)} \rightarrow \{in, out\}$, where
$par(a) = \{ p \in A \mid (p,a) \in L\}$ is the set of \textbf{parents} of an argument $a$.
\end{definition} 

Due to the fact that the set of links can be inferred from the conditions, we will write simply $(A,C)$ to denote an ADF. 
The basic \enquote{building blocks} of the extension--based ADF semantics from \cite{polberg-stairs14,report:semanticsrev} 
are the decisively in interpretations and the derived various types of evaluations.
A two--valued interpretation is simply a mapping that assigns truth values $\{\tvt, \tvf\}$ to (a subset of) arguments.  
For an interpretation $v$, $v^x$ is the set of elements mapped to $x \in \{\tvt, \tvf\}$ by $v$.
A decisive interpretation $v$ for an argument $a \in A$ represents an assignment for a set of arguments $\ext \subseteq A$
s.t. independently of the status of the arguments in $A \setminus \ext$, the outcome of 
the condition of $a$ stays the same.

\begin{definition}
Let $A$ be a collection of elements, $\ext \subseteq A$ its subset and $v$ a two--valued interpretation defined on $\ext$.
A \textbf{completion} of $v$ to a set $Z$ where $\ext \subseteq Z \subseteq A$, is
an interpretation $v'$ defined on $Z$ in a way that $\forall a \in \ext \; v(a) = v'(a)$. $v'$ is a $\tvt/\tvf$ completion of $v$
iff all arguments in $Z \setminus \ext$ are mapped respectively to $\tvt/\tvf$. 
\end{definition}

\begin{definition}
Let $\ADF = (A,L,C)$ be an ADF, $\ext \subseteq A$ a set of arguments and $v$ a two--valued interpretation defined on $\ext$.
$v$ is \textbf{decisive} for
an argument $s \in A$ iff for any two completions 
$v_{par(s)}$ and $v'_{par(s)}$ of $v$ to $\ext \cup par(s)$,  
it holds that $v_{par(s)}(C_s) =v'_{par(s)}(C_s)$.
$s$ is \textbf{decisively out/in} w.r.t. $v$ if $v$ is decisive and all of its completions evaluate $C_s$ to respectively $out,in$.
\end{definition}

From now on we will focus on the minimal interpretations, i.e. those in which both 
$v^{\tvt}$ and  $v^{\tvf}$ are minimal w.r.t. $\subseteq$.
By $min\_dec(x,s)$ we denote the set of minimal two--valued interpretations
that are decisively $x$ for $s$, where $s$ is an argument and $x\in\{in, out\}$. From the positive parts
of a decisively in interpretation for $a$ we can extract arguments required for the acceptance of $a$.
With this information, we can define various types of evaluations, not unlike the powerful sequences in AFNs.
However, due to the fact that ADFs are more expressive than AFNs, it is also the $\tvf$ parts of the used interpretations
that need to be stored \cite{polberg-stairs14,report:semanticsrev}: 

\begin{definition}
Let $\ADF = (A,L,C)$ be an ADF and $\ext \subseteq A$ a set of arguments. 
A \textbf{positive dependency function} (pd--function) on $\ext$ is a function $pd_{\ext}^{\ADF}$ assigning every argument 
$a \in \ext$ an interpretation $v \in min\_dec(in, a)$ s.t. $v^t \subseteq \ext$, or $\mathcal{N}$ for null iff no such
$v$ can be found. $pd_{\ext}^{\ADF}$ is \textbf{sound} on $\ext$ iff for no $a \in \ext$, $pd_{\ext}^{\ADF}(a) =\mathcal{N}$.
$pd_{\ext}^{\ADF}$ is \textbf{maximally sound} on $\ext$ iff it is sound on $\ext' \subseteq \ext$ and there is no other sound
function ${pd'}_{\ext}^{\ADF}$ on $\ext''$ s.t. $\forall a \in \ext'$, $pd_{\ext}^{\ADF}(a) = {pd'}_{\ext}^{\ADF}(a)$, 
where $\ext' \subset \ext'' \subseteq \ext$.
\end{definition} 
 
\begin{definition}
Let $\ADF = (A,L,C)$ be an ADF, $S \subseteq A$ and $pd_{\ext}^{\ADF}$ a maximally sound pd--function of $S$
defined over $\ext \subseteq S$.
A \textbf{partially acyclic positive dependency evaluation} based on $pd_{\ext}^{\ADF}$ for an argument $x \in \ext$ is a triple $(F,(a_0,...,a_n),B)$,
where $F \cap \{a_0,...,a_n\} = \emptyset$, $(a_0,...,a_n)$ is a sequence of distinct elements of $\ext$ satisfying the requirements:
\begin{inparaenum}[\itshape i\upshape)]
\item if the sequence is non--empty, then $a_n = x$; otherwise, $x \in F$
\item $\forall_{i=1}^n, \, pd_{\ext}^{\ADF}(a_i)^{\tvt} \subseteq F \cup \{a_0,...,a_{i-1}\}$, $pd_{\ext}^{\ADF}(a_0)^{\tvt} \subseteq F$
\item $\forall a\in F, \, pd_{\ext}^{\ADF}(a)^\tvt \subseteq F$
\item $\forall a \in F, \exists b \in F$ s.t. $a \in pd_{\ext}^{\ADF}(b)$.
\end{inparaenum}
Finally, $B = \bigcup_{a\in F} pd_{\ext}^{\ADF}(a)^\tvf \cup \bigcup_{i=0}^n \, pd_{\ext}^{\ADF}(a_i)^{\tvf}$.
We refer to $F$ as the \textbf{pd--set}, to $(a_0,...,a_n)$ as the \textbf{pd--sequence} and to $B$ as the \textbf{blocking set} of the evaluation.
A partially acyclic evaluation $(F,(a_0,...,a_n),B)$ for an argument $x \in \ext$
is an \textbf{acyclic positive dependency evaluation} for $x$ iff $F = \emptyset$.
\end{definition}
 
We will use the shortened notation $((a_0,...,a_n),B)$ for the acyclic evaluations.  
There are two ways we can \enquote{attack} an evaluation. Either we accept an argument that needs to be rejected
 (i.e. it is in the blocking set), or we are able to discard one that needs to be accepted (i.e. is in the the pd--sequence or the pd--set).
We will be mostly concerned with the first type. We can now define various discarded sets in ADFs\footnote{The presented definitions are generalizations of the ones from \cite{polberg-stairs14,report:semanticsrev}.}:

\begin{definition}
Let $\ADF = (A,L,C)$ be an ADF and $\ext \subseteq A$ a set of arguments.
The \textbf{standard discarded} set of $\ext$ is
$\ext^+ = \{ a\in A \mid$ for every partially acyclic evaluation $(F,G,B)$ for $a$, $B \cap \ext \neq \emptyset \}$.
The \textbf{partially acyclic discarded} set of $\ext$ is
$\ext^{p+} = \{a \in A \mid$ there is no partially acyclic evaluation $(F',G',B')$ for $a$ 
s.t. $F' \subseteq \ext $ and $B' \cap \ext = \emptyset \}$. 
The \textbf{acyclic discarded} set of $\ext$ is
$\ext^{a+} = \{ a\in A \mid$ for every pd--acyclic evaluation $(F,B) $ for $a, \, B \cap \ext \neq \emptyset \}$.
\end{definition}

Given a set of arguments $\ext$ and its discarded set $S$, we can build a special interpretation -- called \textbf{range} -- with which
we can check for decisiveness. The range can be constructed by assigning $\tvt$ to arguments in $\ext$ and $\tvf$ to those in $S\setminus \ext$.
Under certain conditions $\ext$ and $S$ are disjoint, which brings us to the conflict--free semantics:

\begin{definition}
Let $\ADF = (A,L,C)$ be an ADF.
A set  $\ext \subseteq A$ is a \textbf{conflict--free extension} of $\ADF$ if for all 
$s \in \ext$ we have $C_s (\ext \cap par(s )) = in$. $\ext$ is a
 \textbf{pd--acyclic conflict--free extension} of $\ADF$ iff every $a \in \ext$ has an acyclic evaluation $(F,B)$ on $\ext$ s.t. $B \cap \ext = \emptyset$. 
\end{definition}

\begin{lemma}
\label{lemma:disc}
Let $\ADF = (A,L,C)$ be an ADF and $\ext \subseteq A$ a set of arguments.
If $\ext$ is conflict--free in $\ADF$, then $\ext \cap \ext^+ = \emptyset$ and $\ext \cap \ext^{p+} = \emptyset$. 
Moreover, it holds that $\ext^+ \subseteq \ext^{p+} \subseteq \ext^{a+}$.
If $\ext$ is pd--acyclic conflict--free, then $\ext \cap \ext^{a+} = \emptyset$ and $\ext^{p+} = \ext^{a+}$.
\end{lemma}


By combining a given type of a discarded set and a given type of conflict--freeness, we have developed various families of 
extension--based semantics \cite{polberg-stairs14,report:semanticsrev}. 
We have classified them into the four main types and used an $xy-$ prefixing system to denote them. 
In the context of this work, three of the families will be relevant. 
We will now recall their definitions refer the reader to \cite{report:semanticsrev} for proofs and further explanations.

\begin{definition}
Let $\ADF = (A,L,C)$ be an ADF. Let $\ext \subseteq A$ be a set of arguments and $v_\ext$, $v_\ext^a$ and $v_\ext^p$
its standard, acyclic and partially acyclic ranges. 

If $\ext$ is conflict--free and every $e \in \ext$ is decisively in w.r.t. $v_\ext$ ($v_\ext^p$),
then $\ext$ is \textbf{cc--admissible (ca$_2$--admissible)} in $\ADF$.
If $\ext$ is pd--acyclic conflict--free and every $e \in \ext$ is decisively in w.r.t. $v_\ext^a$,
then $\ext$ is \textbf{aa--admissible} in $\ADF$.

If $\ext$ is cc--admissible (ca$_2$--admissible, aa--admissible) and every argument $e \in A$
decisively in w.r.t. $v_\ext$ ($v_\ext^p$, $v_\ext^a$) is in $\ext$, then $\ext$ is \textbf{cc--complete (ca$_2$--complete, aa--complete)}
in $\ADF$.
If $\ext$is  maximal w.r.t. set inclusion xy--admissible extension, where $x,y \in \{a,c\}$, then it is an \textbf{xy--preferred} extensions of $\ADF$.

If $\ext$ is conflict--free and for every $a \in A \setminus \ext$, $C_a(\ext \cap par(a)) = out$, then $\ext$
is a \textbf{model} of $\ADF$.
If $\ext$ is pd--acyclic conflict--free and $\ext^{a+} = A \setminus \ext$, then $\ext$ is a \textbf{stable} extension of $\ADF$ .

If $\ext$ is the least w.r.t. $\subseteq$ cc--complete extension, then it is the \textbf{grounded} extension of $\ADF$.
If $\ext$ is the least w.r.t. $\subseteq$ aa--complete extension, then it is the \textbf{acyclic grounded} extension of $\ADF$.
\end{definition}

Finally, we can define the two important ADF subclasses. The bipolar ADFs consist only of links that are supporting or attacking. 
This class is particularly valuable due to its computational complexity properties \cite{strass-wallner14complexity}. The other
subclass, referred to as AADF$^+$, consists of ADFs in which our semantics classification collapses. By this we understand
that e.g. every cc--complete extension is aa--complete and vice versa. Moreover, this class provides a more precise correspondence 
between the extension and labeling--based semantics for ADFs \cite{report:semanticsrev}. This means that these frameworks,
 we can use the DIAMOND software \cite{Ellmauthaler2014a} and other results for the labeling--based semantics
\cite{strass-wallner14complexity,strass13instantiating}.

\begin{definition}
\label{def:badfaadf}
Let $\ADF = (A,L,C)$ be an ADF. A link $(r,s) \in L$ is:
\begin{inparaenum}[\itshape i\upshape)]
\item supporting iff for no $R\subseteq par(s)$ we have that $C_s(R) = in$ and $C_s(R \cup \{r\}) = out$
\item attacking iff for no $R\subseteq par(s)$ we have that $C_s(R) =out$ and $C_s(R \cup \{r\}) =in$
\end{inparaenum}
$\ADF$ is a \textbf{bipolar} ADF (BADF) iff it contains only links that are supporting or attacking.
$\ADF$ is a \textbf{positive dependency acyclic} ADF (AADF$^+$) iff 
every partially acyclic evaluation $(F,G,B)$ of $\ADF$ is acyclic.  
\end{definition}

\begin{theorem}
Let $\ADF = (A,L,C)$ be an AADF$^+$. The following holds:
\begin{itemize}
\item Every conflict--free extension of $\ADF$ is pd--acyclic conflict--free in $\ADF$
\item Every model of $\ADF$ is stable in $\ADF$ 
\item The aa/cc/ca$_2$--admissible extensions of $\ADF$ coincide
\item The aa/cc/ca$_2$--complete extensions of $\ADF$ coincide
\item The aa/cc/ca$_2$--preferred extensions of $\ADF$ coincide
\item The grounded and acyclic grounded extensions of $\ADF$ coincide
\end{itemize}
\label{thm:extcollapsetrim}
\end{theorem}  
\begin{example}
\label{ex:adf}
Let us consider the framework is $\ADF= (\{a,$ $b,c,d,e,f,g\}, \{C_a: \top, C_b: \neg a \lor c, C_c:\neg d \lor b, C_d:\top, 
C_e:\neg b, C_f:\neg e, C_g:\neg f\})$. We can observe that both $a$ and $d$ have trivial acyclic evaluations
$((a), \emptyset)$ and $((d), \emptyset)$. For $e$, $f$ and $g$ we can construct $((e), \{b\})$, 
$((f), \{e\})$ and $((g), \{f\})$. The situation only gets complicated with $b$ and $c$;
we have the acyclic evaluations $((b), \{a\})$, $((c,b), \{d\})$, $((c), \{d\})$, $((b,c), \{a\})$
and the partially acyclic one $(\{b,c\}, \emptyset)$.
We can observe that $\emptyset$ is an admissible extension of any type; all of its discarded sets are empty. Decisively
in w.r.t. its ranges are thus $a$ and $d$. The set $\{a,d\}$ is again admissible. Its standard discarded set is $\emptyset$,
however, the acyclic and partially acyclic ones are $\{b,c\}$. Therefore, $\{a,d\}$ is only cc--complete.
Discarding $b$ leads to the acceptance of $e$ and $g$. Hence, $\{a,d,e,g\}$ is an aa and ca$_2$--complete extension, 
though it does not even qualify as a cc--admissible set.
We can now consider the set $\{a,b,c,d\}$. It is conflict--free, but not pd--acyclic conflict--free. Its standard
and partially acyclic discarded set is $\{e\}$, which means that $f$ can be accepted. Hence, $\{a,b,c,d,f\}$
is cc and ca$_2$--complete. Thus, in total we obtain two cc--complete, one aa--complete and two ca$_2$ complete sets.
Our grounded and acyclic grounded extensions are $\{a,d\}$ and $\{a,d,e,g\}$ respectively. 
The latter set is also the only stable extension of our framework. However, both $\{a,d,e,g\}$ and
$\{a,b,c,d,f\}$ are models. 
\end{example}

\subsection{Conceptual Differences Between ADFs and Other Frameworks}

The more direct descendants
of the Dung's framework 
explicitly state \enquote{this is a supporter}, \enquote{this is an attacker} and so on. Thus, in 
order to know
if a given argument can be accepted along with the other arguments, i.e. whether it is attacked, defeated or receives sufficient support, 
we need to go through all
the relations it is a target of. In contrast, the acceptance conditions \enquote{zoom out} from singular 
relations. They tell us whether the argument can be accepted or not w.r.t. a given set of arguments in
a straightforward manner.
The focus is put on what would usually be seen as a target of a relation, while in other frameworks the attention is on the relation source. 
As a consequence, in order to say
if a parent of an argument is its supporter, attacker or none of these, we need analyze the condition further, as
seen in e.g. Def. \ref{def:badfaadf}.
This is also one of the reasons why finding support cycles in ADFs is more difficult
than in other support frameworks. Finally, since the role of parent is derived from how it affects the behavior of an argument, not whether it 
is in e.g. the support relation $N$, an attacker or a supporter in a given framework may not
have the same role in the corresponding ADF:

\begin{example}
\label{ex:afnprob1}
Let $(\{a,b,c\}, \{(b,a), (a,c)\}, \{(\{b\},a)\})$ be an AFN,  
where the argument $a$ is at the same time supported and attacked by $b$. 
In a certain sense, the $(a,b)$ relation is difficult to classify as positive or negative. Although
$a$ cannot be accepted, it is still a valid attacker that one needs to defend from. In the 
ADF setting, the acceptance condition of $a$ is unsatisfiable -- whether we include or exclude
$b$, we always reject $a$. It can also be seen as a $b \land \neg b$ formula. $a$ does not possess any type of an evaluation
and will always end up in any type of a discarded set.
This also means that we do not have to \enquote{defend} from it.  In this particular example, the set $\{c\}$ would not be considered admissible in our AFN,
but it would be considered an admissible extension of any type in the ADF $(\{a,b,c\}, \{C_a = b \land \neg b, C_b = \top, C_c = \neg a\})$.
\end{example}

Thus, there is an important difference between the design of ADFs
and other argumentation frameworks. If we were to represent the situation as a propositional formula, it is like comparing an atom based and a literal
based evaluation. The same issue arises when we consider standard and ultimate versions of logic programming semantics,
as already noted in \cite{report:strass}.
This means that if we want to translate e.g. an AFN into an ADF while still preserving the behavior of the semantics, 
we need to make sure that no argument is at the same time an attacker and a supporter of the same argument.  
A similar issue also appears in the extended argumentation frameworks. 
The defense attack is a type of a positive, indirect relation towards the \enquote{defended} argument. The difference is that 
while in the first case it is also a negative relation towards the argument carrying out the attack, in the latter the attacker and the 
defense attacker might be unrelated. It is not unlike what is informally referred
to as the \enquote{overpowering support} in ADFs. A typical example is a condition of the form $C_a = \neg b \lor c$, 
where $b$ has the power to $out$ the condition unless $c$ is present. Therefore, defense attackers
from EAFC become directly related to the arguments they \enquote{protect} in ADFs, which can lead to inconsistencies. 

\begin{definition}
Let $\AFN = (A, R, N)$ be an AFN and $a$ an argument in $A$. By $N(a) = \{b \mid \exists B \subseteq A$ s.t. $b \in B, BNa\}$
and $R(a) = \{b \mid b R a\}$ we denote the sets of arguments supporting and attacking $a$. 
Then $a$ is \textbf{strongly consistent} iff $N(a) \cap R(a) = \emptyset$. $\AFN$ is strongly consistent
iff all of its arguments are strongly consistent.
%

Let $\EAFC = (A,R,D)$ be an EAFC. $\EAFC$ is \textbf{strongly consistent} iff there is no 
$x,y,z \in A$ and $\ext \subseteq A$ s.t. $(x,y) \in R$, $x \in \ext$ and $(\ext,(z,y)) \in D$.
\end{definition}

Any AFN can be made strongly consistent with the help of no more than
$\left\vert{A}\right\vert$ arguments. We basically introduce extra arguments that take over the support links
leading to inconsistency and connect them to the original sources of these relations. 
Similar technique can be used in translations for EAFCs. 
Unfortunately, due to the space restrictions, 
we cannot focus on this approach here.

Please note that this analysis does not in any way imply that a given $(a,b)$ link is assigned a single permanent \enquote{role} in 
ADFs, such as \enquote{attack} or \enquote{support}.
The framework is flexible and a link can be positive on one occasion an negative on another.
 A more accurate description is that a link (or its source) should have a defined role \enquote{at a point}, i.e. w.r.t. a 
given set of arguments. ADFs ensure consistency, not constancy. 

\section{Translations}
\label{sec:trans}

In this section we will show how to translate the recalled frameworks to ADFs. We will provide both functional and propositional  
acceptance conditions. For the latter, we would like to introduce the following notations.
For a set of arguments $X = \{x_1,...,x_n\}$, we will abbreviate the formula $x_1 \land ... \land x_n$ with $\bigwedge X$
and $\neg x_1 \land ... \land \neg x_n$ with $\bigwedge \neg X$. Similarly, $x_1 \lor ... \lor  x_{n}$ and 
$\neg x_1 \lor ... \lor \neg x_{n}$ will be shortened to $\bigvee X$ and $\bigvee \neg X$. 

\subsection{Translating SETAFs and AFs into ADFs}

A straightforward translation from AFs to ADFs has already been introduced in \cite{tofix:newadf}. 
Let $a \in A$ be an argument and $\{a\}^- = \{x_1,..,x_n\}$ its attacker set in an AF. 
Whenever any of ${x_i}'s$ is present, $a$ cannot be accepted. Only when all of them are absent,
we can assume $a$. The SETAF translation is quite similar. Let $\{a\}^- = \{X_1,...,X_n\}$ be
 the collection of all sets that attack an argument $a$, i.e. sets s.t. $X_i R a$. 
Only the presence of all members of any $X_i$, not just some of them, renders
$a$ unacceptable.
Therefore given any set of arguments that does not fully include at least one attacking set,
the acceptance condition of $a$ is $in$. 
This brings us to the following two translations:

\begin{translation}
\label{trans:dung}
Let $\AF = (A, R)$ be a Dung's framework.
The ADF corresponding to $\AF$ is $\ADF^{\AF} = (A, R, C)$, where $C = \{C_a\}_{a \in A}$
and every $C_a$ is as follows:
\begin{itemize}
\item Functional form: $C_a(\emptyset) = in$ and for all nonempty $B\subseteq \{a\}^-$, $C_a(B) = out$. 
\item Propositional form: $C_a= \bigwedge \neg \{a\}^-$. In case $\{a\}^-$ is empty, $C_a = \top$.
\end{itemize} 
\end{translation} 

\begin{translation}
\label{trans:setafadf}
Let $\SETAF = (A, R)$ be a SETAF. The ADF
corresponding to $\SETAF$ is $\ADF^{\SETAF} = (A, L, C)$, where
$L = \{(x,y) \mid \exists B \subseteq A, x \in B$ s.t. $BR y\}$, 
$C = \{C_a\}_{a \in A}$
and every $C_a$ is created in the following way:  
\begin{itemize}
\item Functional form: for every $B \subseteq \bigcup \{a\}^-$, if $\exists X_i \in \{a\}^-$ s.t. $X_i \subseteq B$, then $C_a(B) =out$; otherwise,
$C_a(B) = in$.
\item Propositional form: $C_a= \bigvee \neg X_1 \land ... \land \bigvee \neg X_n$. If $\{a\}^-$ is empty, $C_a = \top$.
\end{itemize} 
\end{translation} 

Neither AFs nor SETAFs not rely on any form of support.
Therefore, their associated ADFs are both AADF$^+$s and BADFs.
Consequently, our semantics classification collapses and it does not
matter which type of ADF semantics we work with. 

\begin{restatable}{theorem}{thmsetafadfnf}
\label{thm:setafadfnf}
Let $\SETAF = (A,R)$ be a SETAF or AF and $\ADF^{\SETAF} = (A,L,C)$ its corresponding ADF.
Then $\ADF^{\SETAF}$ is an AADF$^+$ and a BADF. 
\end{restatable}  

\begin{restatable}{theorem}{thmsetafadf}
\label{thm:setafadf}
Let $\SETAF = (A,R)$ be a SETAF or AF and $\ADF^{\SETAF} = (A,L,C)$ its corresponding ADF.
A set of arguments $\ext \subseteq A$ is a conflict--free extensions of $\SETAF$ iff it is (pd--acyclic) conflict--free in $\ADF^{\AF}$.
$\ext \subseteq A$ is a stable extensions of $\SETAF$ iff it is (stable) model of $\ADF^{\AF}$.
$\ext \subseteq A$ is a grounded extensions of $\SETAF$ iff it is (acyclic) grounded in $\ADF^{\AF}$.
$\ext \subseteq A$ is a $\sigma$ extensions of $\SETAF$, where
where $\sigma \in \{$admissible, preferred, complete$\}$ iff it is an
xy--$\sigma$--extension of $\ADF^{\AF}$ for $x,y \in \{a,c\}$. 
\end{restatable}  
\begin{example}
Let us continue Example \ref{exa:setaf}. The ADF associated with $\SETAF$ 
is $\ADF^{\SETAF} = (\{a,b,c,d,e\},
\{C_a: \neg a \land \neg e,
C_b : \neg b, C_c : \neg a, C_d : \neg c, C_e : \neg b \lor \neg d \})$. $\emptyset$ is an admissible extension
of any type; its discarded set is also empty. We can observe that $\{c,e\}$ is conflict--free in $\ADF^{\SETAF}$. Its discarded set is $\{a,d\}$,
thus making the set admissible in $\ADF^{\SETAF}$. No other argument is decisively in w.r.t. the produced ranges
and thus both sets are also complete. This makes $\emptyset$ the grounded and $\{c,e\}$ the preferred extension.
Since $b$ is not contained in any discarded set, $\ADF^{\SETAF}$ has no stable or model extensions.
\end{example}

\subsection{Translating EAFCs into ADFs}

We can now focus on translating EAFCs into ADFs. Let us assume we have an attack $(b,a)$ that is defense attacked
by sets $\{c,d\}$ and $\{e\}$. We can observe that $a$ is rejected only if $b$ is present and none of the defense attacking
sets is fully present. On the other hand, if $b$ is not there or either $\{c,d\}$ or $\{e\}$ are accepted, then
the requirements for $a$ are satisfied. Therefore, for a given EAFC, we can create an ADF in the following way:

\begin{translation} 
\label{trans:eafcadfcons}
Let $\EAFC = (A, R, D)$ be a strongly consistent EAFC. Its corresponding ADF is $\ADF^{\EAFC} = (A,L,C)$, where
$L = \{ (a,b) \mid a R b$ or $\exists c \in A, \ext \subseteq A$ s.t. $a \in \ext, (\ext, (c,b)) \in D\}$, 
$C = \{C_a \mid a\in A\}$ and every $C_a$ is as follows:
\begin{itemize}
\item Functional form: for every set $B \subseteq par(a)$, if $\exists x \in B$ s.t. $(x,a) \in R$ and $\nexists B' \subseteq B$ s.t. $(B',(x,a)) \in D$, then $C_a(B) = out$;
otherwise, $C_a(B) = in$ 
\item Propositional form: if $\{a\}^- = \emptyset$, then $C_a = \top$; otherwise, $C_a = \bigwedge_{b \in A, (b,a)\in R} att_a^b$, where
$att_a^b = \neg b \lor (\bigwedge B_1 \lor ... \bigwedge B_m)$ and $D_{b,a} = \{B_1,...,B_m\}$ is the collection of all sets 
$B_i \subseteq A$ s.t. $(B, (b,a)) \in D$. If $D_{b,a}$ is empty, then $att_a^b = \neg b$.
\end{itemize}
\end{translation}

Although EAFCs are more advanced than e.g. AFs, their associated ADFs are still bipolar. However, only in the case
of bounded hierarchical EAFCs they are also AADF$^+$s. 
The EAFC semantics are now connected to the ca$_2$--semantics family.
Since the ADF associated with the framework from Example \ref{ex:afraeaf} is precisely the one we have
considered in Example \ref{ex:adf}; we refer the reader there for further details.

\begin{restatable}{theorem}{thmeafcadfnf}
\label{thm:eafcadfnf}
Let $\EAFC = (A,R,D)$ be a strongly consistent EAFC and $\ADF^{\EAFC} =(A,L,C)$ its corresponding ADF.
$\ADF^{\EAFC}$ is a BADF. 
If $\EAFC$ is bounded hierarchical, then $\ADF^{\EAFC}$ is an AADF$^+$.
\end{restatable}

\begin{restatable}{theorem}{thmeafcadf}
\label{thm:eafcadf}
Let $\EAFC$ be a strongly consistent EAFC and $\ADF^{\EAFC} = (A,L,C)$ its corresponding ADF.
A set of arguments $\ext \subseteq A$ is a conflict--free extension of $\EAFC$
iff it is conflict--free in $\ADF^{\EAFC}$. $\ext$ is a stable extension of $\EAFC$
iff it is a model of $\ADF^{\EAFC}$. $\ext$ is a grounded extension of $\EAFC$
iff it is the acyclic grounded extension of $\ADF^{\EAFC}$. 
Finally, $\ext$ is a $\sigma$--extension of $\EAFC$, where $\sigma \in \{$admissible, complete, preferred$\}$,
 iff it is a $ca_2$--$\sigma$--extension of $\ADF^{\EAFC}$.
\end{restatable}


\subsection{Translating AFNs into ADFs}

In order to accept an AFN argument, two conditions need to be met. First of all, just like in AFs, the attackers
of a given argument need to be absent. However, in addition, at least one member of every supporting set
needs to be present. This gives us a description of an acceptance condition; the acyclicity
will be handled by the appropriate semantics.

\begin{translation}
\label{trans:afn}
Let $\AFN = (A, R, N)$ be a strongly consistent AFN. The corresponding ADF is $\ADF^{\AFN} = (A, L, C)$, 
where $L = \{(x,y) \mid (x,y) \in R$ or  $\exists B \subseteq A, x \in B$ s.t. $BN y\}$, 
$C = \{C_a \mid a\in A\}$ and every $C_a$ is as follows:
\begin{itemize}
\item Functional form: for every $P' \subseteq par(a)$, if $\exists p \in P'$ s.t. $p R a$ or $\exists Z \subseteq A$ s.t. $Z N a$ and $Z \cap P' = \emptyset$,
then $C_a(P') = out$; otherwise, $C_a(P') = in$.
\item Propositional form: $C_a = att_a \cap sup_a$, where:
\begin{itemize}
\item  $att_a = \bigwedge \neg \{a\}^-$ or $att_a= \top$ if $\{a\}^- = \emptyset$
\item  $sup_a = (\bigvee Z_1 \land ... \land \bigvee Z_m)$, where $Z_1,...,Z_m$ are all subsets of $A$ s.t. $Z_i N a$, 
or $sup_a = \top$ if no such set exists
\end{itemize} 
\end{itemize}
\end{translation}

The produced ADFs are still bipolar. However, whether a given ADF is an AADF$^+$ or not, depends
on the support relation in the source AFN.

\begin{restatable}{theorem}{consafnadfnf}
\label{thm:consafnadfnf}
Let $\AFN = (A, R, N)$ be a strongly consistent AFN and $\ADF^{\AFN} =(A,L,C)$ its corresponding ADF. Then $\ADF^{\AFN}$ is a BADF.  
\end{restatable}  

The AFN semantics are built around the notion of coherence, which requires all relevant arguments to be (support--wise)
derived in an acyclic manner. Thus, not surprisingly,
it is the aa--family of ADF semantics that will be associated with the AFN semantics. In particular, we can relate powerful sequences
to the acyclic evaluations. This also allows us to draw the connection between the acyclic discarded set in ADFs and the discarded set $\ext^{att}$ in AFNs. 
Hence, there is a correspondence between the defense in AFNs and being decisively in w.r.t. a given interpretation in ADFs.  
This in turns tells us the relation between the extensions of AFNs and ADFs:

\begin{restatable}{lemma}{ppd}
\label{lemma:ppd}
Let $\AFN = (A, R, N)$ be a strongly consistent AFN and $\ADF^{\AFN} =(A,L,C)$ its corresponding ADF.
For a given powerful sequence for an argument $a \in A$ we can construct an associated
pd--acyclic evaluation and vice versa.   
\end{restatable} 
\begin{restatable}{theorem}{afnadfsem}
\label{thm:afnadfsem}
Let $\AFN = (A, R, N)$ be a strongly consistent AFN, $\ADF^{\AFN} =(A,L,C)$ its corresponding ADF.
$\ext$ is strongly coherent in $\AFN$ iff it is pd--acyclic conflict--free in $\ADF^{\AFN}$. 
$\ext$ is a $\sigma$--extension of $\AFN$, where $\sigma \in \{$admissible, complete, preferred$\}$
iff it is an aa--$\sigma$--extension of $\ADF^{\AFN}$.
$\ext$ is stable in $\AFN$ iff it is stable in $\ADF^{\AFN}$.
$\ext$ is grounded in $\AFN$ iff it is acyclic grounded in $\ADF^{\AFN}$.
\end{restatable}

\begin{example}
Let us continue Example \ref{ex:afn}. The ADF associated with our AFN is $(\{a$, $b$, $c$, $d,$ $e,$ $f\}$, $\{C_a: b \lor c$, $C_b:\neg d$,
$C_c:\neg e$, $C_d: \neg f$, $C_e: \neg a, C_f: f\})$. $\emptyset$ is trivially aa--admissible. Its acyclic discarded set is $\{f\}$, 
thus making $d$ decisively in. Hence, $\emptyset$ is not aa--complete. The set $\{d\}$ discards $f$ and $b$. This is not enough
to accept any other argument.
Hence, it is both aa--admissible and aa--complete. The set $\{e\}$ is pd--acyclic conflict--free, but not aa--admissible (it discards
$f$ and $c$). However, $\{d,e\}$ is aa--admissible (discarded set is $\{a,b,f,c\}$) and aa--complete.
We can also show that $\{a,c,d\}$ is aa--admissible and aa--complete (discarded set is $\{b,f,e\}$). 
Therefore, $\{d\}$ is the acyclic grounded extension, while $\{d,e\}$ and $\{a,c,d\}$ are aa--preferred and stable.
\end{example}

\section{Conclusions}

In this paper we have presented a number of translations from different argumentation frameworks to ADFs. We could have
observed that for every structure, we have found a family of ADF semantics which followed similar principles. We have also
identified to which ADF subclass a given translation--produced framework belongs so that the results from 
\cite{strass-wallner14complexity,strass13instantiating,Ellmauthaler2014a} can be exploited. Due to the space constraints, 
we could not have presented certain approaches. In particular, we have omitted
the evidential argumentation systems \cite{inproc:eas}. However, based on the SETAF and AFN
methods and the results from \cite{inproc:easafn}, this approach can be easily extrapolated. We also did not present
the translations removing the strong consistency assumptions, although we hope we will manage to do so in the extended version
of this work. Finally, our work falls into the research on framework intertranslatability
 \cite{article:newbaf,inproc:easafn,Modgil11,inproc:moving,Baronietal2011,BoellaGTV09}. 
However, in this case we are moving from
less to more complex structures, not the other way around. Moreover, the fact that we are working with ADFs
means that the currently established methods are not particularly applicable. To the best of our knowledge, our work
is the first one to focus on establishing the relations between ADFs and other argumentation frameworks.

\newpage 
\bibliographystyle{unsrt}
\bibliography{references,argumentation}

\newpage
\section{Proof Appendix} 

\subsection{Background Appendix}
\subsubsection{Additional notions and proofs for SETAFs}

Due to the space restrictions, the following Theorem was not explicitly defined in the text \cite{incoll:setatt}:
\begin{theorem}
\label{thm:compsetaf} 
Let $\SETAF = (A, R)$ be a SETAF. The following holds:
\begin{itemize}
\item Every preferred extension of $\SETAF$ is a complete extension of $\SETAF$, but not vice versa.
\item The grounded extension of $\SETAF$ is the least w.r.t. $\subseteq$ complete extension of $\SETAF$.
\item The complete extensions of $\SETAF$ form a complete semilattice w.r.t. set inclusion. 
\end{itemize}
\end{theorem}

\subsubsection{Additional notions and proofs for EAFCs}

First of all, we would like to focus on EAFCs. We can observe that the original definition does not require the original arguments to be defeated
with reinstatement.

\begin{definition}
Let $\EAFC = (A,R,D)$ be an EAFC. The set $\ext \subseteq A$ is
a \textbf{stable}  extension of $\EAFC$ iff for every argument $b \notin \ext$, $\exists a \in \ext$ s.t. $a$ defeats$_\ext$ $b$ in $\EAFC$.
\end{definition}

 However, we can observe that if for a given attack from $\ext$ there existed a suitable defense attack, then the set carrying it out
could not have been fully in $\ext$. Otherwise, we would not be dealing with a defeat anymore. Consequently, in every case there is
an argument outside $\ext$, and as it will always be attacked, we can build a reinstatement set for any defeat$_\ext$. This means that our
definition is equivalent.

We would also like to show that Thm. \ref{thm:dungeaf} is true, as to the best of our knowledge it was
not formally proved in any other work.

\thmdungeaf*

\begin{proof}
Let $\ext \subseteq A$ be a preferred extension of $\EAFC$. Assume it is not complete; as $\ext$ is admissible, 
this means that there is an argument $a \in A \setminus \ext$
that is defended by $\ext$. Let us consider the extension $\ext' = \ext \cup \{a\}$. Due to defense, it cannot be the case
that $a$ defeats$_\ext$ any argument in $\ext$ and vice versa. Furthermore, $a$ cannot be defeating itself w.r.t. $\ext$ either.
This means that either there are no relevant conflicts in $R$ to start with, or they are already defense attacked by elements in $\ext$.
In both cases this leads to the conclusion that $\ext'$ is conflict--free. We now need to show it is admissible. Let us consider
an arbitrary defeat$_\ext$ by $b \in \ext$ on $c \in A$ that has a reinstatement set $\{(x_1, y_1),...,(x_n, y_n)\}$ on $\ext$.
As no argument in $\ext$ defeats$_\ext$ $a$, it cannot be the case that there is a pair $(x_i, y_i)$ in the reinstatement set
s.t. $(B, (x_i, y_i)) \in D$, where $B \subseteq \ext \cup \{a\}$. 
Therefore, if $\ext$ defeats an argument $c\in A$ with reinstatement on $\ext$, then so does $\ext'$. 
We can also observe that if an argument $c \in A$ did not defeat$_\ext$ any argument in $\ext$, then it does not defeat$_{\ext'}$
any argument in $\ext'$ either. This brings us to the result that $\ext'$ has to be admissible. This however means that $\ext$ could not
have been a maximal admissible extension -- we can observe that $\ext \subset \ext'$ -- and thus we contradict the assumption it is preferred.
Hence, we can conclude that if $\ext$ is preferred, then it is complete. 

In order to show that not every complete
extension is preferred, it suffices to look at a Dung--style EAFC 
$(\{a,b,c,d,e\}, \{(a,b)$, $(c,b)$, $(c,d)$, $(d,c)$, $(d,e)$, $(e,e)\}, \emptyset)$. It has three complete extensions -- $\{a\}$,
$\{a,c\}$ and $\{a,d\}$ -- and only two of them are preferred.


Let $\ext \subseteq A$ be a stable extension of $\EAFC$. We can observe it is also admissible in $\EAFC$: every argument outside of $\ext$
is defeated$_\ext$ by $\ext$ and the collection of all defeats$_\ext$ carried out by elements of $\ext$ is a simple reinstatement set
for any of them. Therefore, every argument $a \in \ext$ is defended by $\ext$, and $\ext$ is admissible.
Due to conflict--freeness of $\ext$ it cannot be the case
that at the same time, $\ext$ defeats$_\ext$ and defends an argument $b \notin \ext$. Therefore, $\ext$ is complete in $\EAFC$.
The fact that not every complete extension is stable
can be observed in the aforementioned example; only the set $\{a,d\}$ is stable in that particular framework.

In order to show that the grounded extension is a minimal complete one, we will use the operator iteration approach.
Assume $\ext$ is the grounded extension and there exists a smaller complete extension $\ext' \subset \ext$.  
Let $G = \emptyset$. We can observe that only those arguments that are not
attacked in $R$ at all can be acceptable w.r.t. $\emptyset$ -- there is no argument in $G$ that would prevent an attack turning
into a defeat. Therefore, if an argument $b \in A$ is acceptable w.r.t. $\emptyset$, then it is acceptable w.r.t. any other set of arguments.
Thus, we can add the arguments produced by $\charEAFC(\emptyset)$ to $G$ and observe that $G \subseteq \ext' \subset \ext$ due to
the completeness of $\ext'$.

Let us now apply the operator again and let $a \in A$ be an argument acceptable w.r.t. $G$. Assume it is not acceptable w.r.t. $\ext'$.
This means there is an argument $b \in A$ that defeats$_{\ext'}$ $a$ and is not in turn defeated$_{\ext'}$ by any argument
$c \in \ext'$ with a reinstatement set. We can observe that if $b$ defeats$_{\ext'}$ $a$, then due to the fact that $G \subseteq \ext'$,
$b$ defeats$_{G}$ $a$ as well. Therefore, $G$ has to defeat$_{G}$ $b$ with a reinstatement set on $G$, even though it is not the
case for $\ext'$. 
Let $c \in G$ be an argument carrying out the reinstated defeat on $b$ in $G$ and let $\{(x_1, y_1),...,(x_n, y_n)\}$
be the relevant reinstatement set. We will show that $G' = \charEAFC(G)$ also defeats$_{G'}$ $b$ with the same reinstatement.
We can observe that every argument defense attacking any of the defeats listed in the reinstatement set is defeated$_G$ by $G$.
Therefore, it cannot be acceptable w.r.t. $G$ and will not appear in $G'$. This means that any pair in the reinstatement set
that was a defeat$_G$ is also a defeat$_{G'}$. We can therefore show that if $G$ defeats$_G$ $b$ with a reinstatement,
then so does the grounded extension of $\EAFC$ (which in this case, is $\ext$).
Now, if $c$ does not defeat$_{\ext'}$ $b$, then there is an argument $d \in \ext'$
s.t. $(d,(c,b)) \in D$. Consequently, $d$ has to be defeated$_G$ by $G$ with a reinstatement, which based on the previous
explanations means that $d$ cannot be in the grounded extension. Therefore, $\ext'$ cannot be a subset of $\ext$ and we
reach a contradiction. This brings us to the conclusion that $a$ has to be acceptable w.r.t. $\ext'$
and by completeness of $\ext'$, it holds that $G \subseteq \ext' \subset \ext$ where $G$ is extended by the arguments
in $\charEAFC(G)$. 

We can continue this line of reasoning till our grounded extension is computed and conclude that $G \subseteq \ext' \subset \ext = G$.
We thus reach a contradiction with the assumption that $\ext' \subset \ext$ and can therefore conclude that $\ext$ has to be a minimal
complete extension of $\EAFC$. The fact it is not necessarily the least can be observed in Example \ref{ex:afraeaf}.
\end{proof}

\begin{theorem}
\label{thm:cdefrei} 
Let $\EAFC = (A,R,D)$ be a finite EAFC and $\ext \subseteq A$ be a conflict--free extension of $\EAFC$. If an argument $a \in \ext$ defeats$_\ext$
an argument $b \in A$, then there is no reinstatement set for this defeat$_\ext$ on $\ext$
iff there exists a sequence $((Z_1, (x_1, y_1)), ... ,$ $(Z_n, (x_n, y_n)))$ of distinct defense attacks from $D$ s.t.:
\begin{itemize}
\item there is a set of arguments argument $G \subseteq A$ s.t. $x_n = a$, $y_n = b$ and $Z_n = g$ 
\item no two pairs $(x_i, y_i)$ and $(x_j, y_j)$ are the same for $i \neq j$
\item for every  $(Z_i, (x_i, y_i))$ where $1< i \leq n$,
either no argument $h$ in $\ext$ defeats$_\ext$ any argument $z \in Z_i$ or for every such defeat, there exists a set $L \subseteq A$ s.t.
$(L, (h,z)) \in \{(Z_1, (x_1, y_1)), ... , (Z_{i-1}, (x_{i-1}, y_{i-1}))\}$, and 
\item no argument in $\ext$ defeats$_\ext$ any argument in $Z_1$.
\end{itemize}
\end{theorem}

\begin{proof}
Let $(x,y) \in R$. By $datt(x,y)$ we denote the set of sets of arguments that carry out defense attacks on $(x,y)$, i.e. $datt(x,y)= \{ C \mid (C,(x,y)) \in D\}$.

Let us first show that if there is no reinstatement set for the $(a,b)$ defeat$_\ext$ on $\ext$, then a suitable sequence 
$((Z_1, (x_1, y_1)), ... ,$ $(Z_n, (x_n, y_n)))$ exists. Due to the fact that no reinstatement set
exists, then $\{ (a,b) \}$ is not a reinstatement set for the the defeat$_\ext$ of $a$ on $b$. 
Hence, $datt(a,b)$ is not empty
and there exists at least one $Z \in datt(a,b)$ s.t. $b \notin Z$ -- otherwise, $\{(a,b)\}$ would have been a reinstatement set. 
From this, we can always choose such a $Z$ s.t. none of its elements is defeated$_\ext$
by $\ext$ or none of such defeats$_\ext$ has a reinstatement set -- otherwise, we could have joined these sets and added $(a,b)$
to obtain a reinstatement set for the $a$--$b$ defeat$_\ext$. 
Let us denote the sets meeting these requirements with $D_1^1,...,D_k^1$. 
If for any of the $D_j^1$, no $d \in D_j^1$ is defeated$_\ext$ by $\ext$,
then $((D_j^1, (a,b)))$ is a sequence meeting our requirements and we are done.

Let us therefore assume that for every $D_j^1$, we can find arguments $x \in D_j^1$, $e \in \ext$ s.t. $e$ defeats$_\ext$
$x$. Again, none of such defeats can have a reinstatement set on $\ext$ -- otherwise, we would have been able to construct
a reinstatement set for $(a,b)$.  
For the same reasons as above, in every $datt(e,x)$ there is a set not containing $x$ which is either not defeated$_\ext$ by 
$\ext$ or no such defeat has a reinstatement set on $\ext$. 
However, we can also observe that if $e=a$, then we can choose such an $D_j^1$ and $x \in D_j^1$
and $D_{mj}^2 \in datt(e, x)$ for $1 < m <\lvert datt(e,x) \rvert$ s.t. that $D_{mj}^2$ meets our requirements and does not contain $b$. 
If it were not possible, then again there would have been a reinstatement set for $(a,b)$.
Thus, we can filter our first and second level $D$'s and continue our analysis. 
If it is the case that any of $D_{mj}^2$ is not defeated$_\ext$ by 
$\ext$, then again $((D_{mj}^2, (e,x)))$ is a satisfactory sequence for the $e$--$x$ defeat. By appending such sequences
for the remaining defeats on $x$ and including the $(D_j^1, (a,b))$ defeat, we can receive the desired sequence for $(D_j^1, (a,b))$.

We can therefore focus again on the case that for no defeat$_\ext$ by any argument $f \in \ext$ on any $D_mj^2$ there is a 
reinstatement set on $\ext$. We can continue the analysis in the similar manner, each time showing that a sequence with unique conflicts
can be built and that for each defense attacks in the sequence is \enquote{protected} by the attacks lower in the sequence. 
Since the amount of conflicts in our framework is finite, we are bound to reach defense attacks by arguments that are not 
defeated$_\ext$ by $\ext$. This concludes this part of the proof. 

Let now $((Z_1, (x_1, y_1)), ... , (Z_n, (x_n, y_n)))$ be a defense attack sequence satisfying our requirements. There is no argument
$d \in \ext$ s.t. $d$ defeats$_\ext$ any argument in $Z_1$. Therefore, there cannot be a reinstatement set for $(x_1, y_1)$.
If there exists an argument in $\ext$ defeating$_\ext$ $Z_2$, then by the construction of the sequence it holds that this conflict
is defense attacked by $Z_1$. Consequently, there cannot be a reinstatement set for this conflict on $\ext$. We can repeat this procedure
till we reach $Z_n$. As there is no defeat$_\ext$ on any element in$Z_n$ that can be reinstated, there is no reinstatement set for $(x_n, y_n)$. 
This concludes the proof.
\end{proof}

\subsubsection{Additional notions and proofs for AFNs}

Let us now focus on AFNs. We can observe that the definitions of defense and stability that we have used in this paper are different
from the ones from \cite{incoll:newafn}:

\begin{definition}
Let $\AFN = (A,R,N)$ be an AFN, $\ext \subseteq A$ and $a \in A$. 
A set $\ext$ \textbf{defends} $a$ in $\AFN$ iff $\ext \cup \{a\}$ is coherent and for each $b \in A$, if $b R a$ then for each 
coherent $C\subseteq A$
that contains $b$, there exist arguments $e \in \ext, c \in C$ s.t. $e R c$.  
\end{definition}

\begin{definition}
Let $\AFN = (A,R,N)$ be an AFN. 
The set of arguments \textbf{deactivated} by $\ext$ is defined by $\ext^+ = \{a \mid \exists e \in \ext $ s.t. $e R a$ or there is a $ 
B \subseteq A $ s.t. $B N a$ and $
\ext \cap B = \emptyset\}$. 
$\ext$ is \textbf{stable} in $\AFN$ iff it is complete in $\AFN$ and $\ext^+ = A\setminus \ext$.
\end{definition}

However, the fact that our notion of defense is equivalent to the original one can be easily proved from the definition
of the discarded set.
\begin{lemma}
Let $\AFN = (A,R,N)$ be an AFN, $\ext \subseteq A$ and $a\in A$. $a$ is defended by $\ext$ in $\AFN$
iff $\ext \cup \{a\}$ is coherent and $\forall b \in A$ s.t. $bRa$, $b \in \ext^{att}$. 
\label{lemma:afndefatt}
\end{lemma} 

The discarded set is a subset of the deactivated set. Using this, we can show that the stable semantics can
be defined with strongly coherent semantics as well, not just complete.

\begin{restatable}{lemma}{lemmaafndeatt}
Let $\AFN = (A,R,N)$ be an AFN and $\ext \subseteq A$ be a strongly coherent set. Then $\ext^{att} \subseteq \ext^+$.
\label{lemma:afndeatt}
\end{restatable}  

\begin{proof}
 Let us assume this is not the case, i.e. an argument $a\in A$
is in $\ext^{att}$, but $\nexists e \in \ext, e R a$ and $\forall C\subseteq A$ s.t. $CNa$, $C\cap \ext \neq \emptyset$. 
It is easy to see that since sufficient support is provided and $\ext$ is coherent, then $\ext \cup \{a\}$ would have to be coherent
as well. Since $a \in \ext^{att}$, every coherent set containing $a$ is attacked by $\ext$. As $\ext$ is also conflict--free, it can 
thus only be the case that $\exists e \in \ext$
s.t. $e R a$. We reach a contradiction. Hence, whatever is in $\ext^{att}$, is also in $\ext^+$.
\end{proof}

\begin{restatable}{lemma}{lemmaafnstbtwo}
\label{lemma:afnstb2}
Let $\AFN = (A,R,N)$ be an AFN. A set $\ext \subseteq A$ is a stable in $\AFN$ iff it is strongly coherent and $\ext^{att} = A\setminus \ext$.
\end{restatable}

\begin{proof}
Let us show that if $\ext$ is strongly coherent and $\ext^{att} = A\setminus \ext$, then $\ext$ is stable. By Lemma \ref{lemma:afndeatt} we know
that $\ext^{att} \subseteq \ext^+$. Thus, it suffices to show that $\ext$ is complete. 
Since $\ext$ is strongly coherent, $\ext \cap \ext^{att} = \emptyset$.
Moreover, from Lemma \ref{lemma:afndefatt} and the fact that $\ext^{att} = A\setminus \ext$ it follows that $\ext$ is at 
least admissible. Now assume there is an argument
$a \notin \ext$ that is defended by $\ext$. Since $a \in \ext^{att}$, $\ext$ could not have been conflict--free in the first place. 
Thus, there cannot be a defended
argument not in $\ext$. Hence, the set is complete and as a result, also stable.

Let us now show the other way. Since $\ext$ is complete, it is at least strongly coherent. What remains to be shown is that in this case, 
whatever is in $\ext^+$
is in $\ext^{att}$. Let us assume it is not the case, i.e. there is an argument in  $a\in \ext^+$ s.t. $\ext$ does not attack all coherent sets containing $a$.
Let $(a_0,...,a_n)$ be a powerful sequence for $a$ that is not attacked by $\ext$. Assume that none of the elements of the sequence belong to $\ext$.
This means that $a_0$ is in $\ext^+$, and as it requires no support due to the powerful sequence conditions, it has to be the case that $\ext$ attacks it. 
Consequently, the powerful sequence for $a$ would also be attacked by $\ext$ and we would reach a contradiction. 
Thus, let us assume that at least $a_0$ is in $\ext$. If $a_1$ is not there, then by the fact it is supported by $a_0$ and thus by 
$\ext$ we again would reach a conclusion
that it can only be the case that $\ext$ attacks $a_1$. Consequently, the sequence would again be attacked and we reach a contradiction. We will come
to the same conclusion when we assume that $a_1$ is in $\ext$, but $a_2$ is not.
We can
continue until we reach $a_n = a$ and it is easy to see that it could not have been the case that $a$ was in $\ext^+$, but not in $\ext^{att}$. Hence,
$\ext$ is strongly coherent and $\ext^{att} = A\setminus \ext$.
\end{proof}

\begin{theorem}
\label{thm:compafn}
Let $\AFN = (A,R,N)$ be an AFN.
The following holds:
\begin{itemize}
\item the grounded extension of $\AFN$ is the least w.r.t.  $\subseteq$ complete extension of $\AFN$
\item a preferred extension of $\AFN$ is a maximal w.r.t. $\subseteq$ complete extension of $\AFN$
\item each stable extension of $\AFN$ is preferred in $\AFN$, but not vice versa.
\end{itemize}
\end{theorem}

\subsubsection{Additional notions and proofs for ADFs}

In this work, we have used shortened versions of various Theorems and notions from \cite{report:semanticsrev}. In particular, we have
trimmed Theorem \ref{thm:extcollapsetrim} to the semantics recalled in this work. We will now reintroduce certain notions
due to their impact on the proofs in the next sections.

\begin{theorem}
Let $\ADF = (A,L, C)$ be an AADF$^+$. The following holds:
\begin{enumerate}
\item Every conflict--free extension of $\ADF$ is pd--acyclic conflict--free in $\ADF$
\item Every model of $\ADF$ is stable in $\ADF$
\item Given a conflict--free set of arguments $\ext \subseteq A$, $\ext^+ = \ext^{p+} = \ext^{a+}$
\item The aa/cc/ac/ca$_1$/ca$_2$--admissible extensions of $\ADF$ coincide
\item The aa/cc/ac/ca$_1$/ca$_2$--complete extensions of $\ADF$ coincide
\item The aa/cc/ac/ca$_1$/ca$_2$--preferred extensions of $\ADF$ coincide
\item The grounded and acyclic grounded extensions of $\ADF$ coincide
\end{enumerate}
\label{thm:extcollapse}
\end{theorem}  
 
\begin{theorem}
Let $\ADF = (A,L,C)$ be an AADF$^+$. The following holds:
\begin{enumerate}
\item Every admissible labeling of $\ADF$ has a corresponding aa/ac/cc/ca$_1$/ca$_2$--admissible extension of $\ADF$ and vice versa.
\item Every complete labeling of $\ADF$ has a corresponding aa/ac/cc/ca$_1$/ca$_2$--complete extension of $\ADF$ and vice versa.
\item Every preferred labeling of $\ADF$ has a corresponding aa/ac/cc/ca$_1$/ca$_2$--preferred extension of $\ADF$ and vice versa. 
\end{enumerate}
\label{thm:labextcollapse}
\end{theorem}

\begin{lemma}
Let $\ADF = (A,C)$ be an ADF and $\ext \subseteq A$ a model of $\ADF$. Then $\ext^{a+} = A\setminus \ext$ and $\ext^{p+} = A \setminus \ext$.
\label{lemma:modrange}
\end{lemma}

\begin{proposition}
Let $\ADF = (A,C)$ be an ADF, $\ext \subseteq A$ a standard and $S\subseteq A$ a pd--acyclic conflict--free extension of $\ADF$, with 
$v_\ext$, $v_\ext^p$, $v_\ext^a$, $v_S$, $v_S^p$ and $v_S^a$ as their corresponding standard, partially acyclic and acyclic
range interpretations. Let $s\in A$ be an argument. 
The following holds: 
\begin{enumerate}
\item If $v_\ext(s) = \tvf$, then $s$ is decisively out w.r.t. $v_\ext$. Same holds or $v_\ext^p$, but not for $v_\ext^a$.
\item If $v_S(s) = \tvf$, then $s$ is decisively out w.r.t. $v_S$. Same holds for $v_\ext^p$ and $v_\ext^a$.

\item If $v_\ext(s) = \tvf$, then $C_s(\ext \cap par(s)) = out$. Same holds or $v_\ext^p$, but not for $v_\ext^a$.
\item If $v_S(s) = \tvf$, then $C_s(S \cap par(s)) = out$. Same holds for $v_\ext^p$ and $v_\ext^a$. 
\end{enumerate}
\label{prop:range}
\end{proposition}

\begin{lemma}{\textbf{CC/AC/AA Fundamental Lemma}:}
Let $\ADF = (A,C)$ be an ADF, $\ext$ a cc(ac)--admissible extension of $\ADF$, $v_\ext$ its range interpretation and $a, b\in A$ 
two arguments decisively in w.r.t. $v_\ext$. Then $\ext' = \ext \cup \{a\}$ is cc(ac)--admissible in $\ADF$
and $b$ is decisively in w.r.t. $v_\ext'$.

Let $\ext$ be an aa-admissible extension of $\ADF$, $v_\ext^a$ its acyclic range interpretation and $a, b\in A$ two arguments decisively in w.r.t. $v_\ext^a$. 
Then $\ext' = \ext \cup \{a\}$ is aa--admissible in $\ADF$ and $b$ is decisively in w.r.t. $v_\ext'^a$.
\label{fund1}
\end{lemma}

\subsection{Translation Appendix}

\subsubsection{Translations for SETAFs}

In this section we will include the proofs of Theorems \ref{thm:setafadfnf} and \ref{thm:setafadf}. As partial results for the latter,
we also introduce Theorem \ref{thm:sacf} and Lemmas \ref{lemma:sarange} and \ref{lemma:sadefense}.

\thmsetafadfnf*
\begin{proof}
SETAFs properly generalized AFs. Therefore, it suffices to show that SETAF--produced ADFs are both AADF$^+$s and BADFs.

Let $(a,b) \in L$ and $\ext \subseteq par(b)$ a subset of parents of $b$ in $\ADF$. From the construction of the condition we can observe
that if $C_b(\ext) = out$, then $C_b (\ext \cup \{a\}) = out$ as well. Therefore, the $(a,b)$ link is attacking. This
holds for every link in $\ADF^{\SETAF}$ and therefore $\ADF^{\SETAF}$ is a BADF.

Let $a \in A$ be an argument. $a$ may have more than one minimal decisively in interpretation, however, 
in all of them the $\tvt$ part is empty and $\tvf$ corresponds to some subset of parents of $a$.
Consequently, $a$ satisfies the $a_0$ requirements of a pd--acyclic evaluation and every partially acyclic evaluation will be acyclic. Hence,
$\ADF^{\SETAF}$ is an AADF$^+$.
\end{proof} 

\begin{restatable}{theorem}{thmsacf}
\label{thm:sacf}
A set of arguments $\ext$ is a conflict--free extension of $\SETAF$ iff it is a conflict--free extension of $\ADF^{\SETAF}$.
\end{restatable}

\begin{proof}
Assume that $\ext$ is a conflict--free extension of $\SETAF$, but not of $\ADF^{\SETAF}$. This means
that there is an argument $e \in \ext$ s.t. $C_e(\ext \cap par(e)) = out$. 
However, by Translation \ref{trans:setafadf}, if 
$\ext \cap par(e)$ is mapped to $out$, then  $\exists \ext' \subseteq \ext$ s.t. $\ext'Re$. Therefore,
$\ext$ cannot be conflict--free in $\SETAF$ and we reach a contradiction.

Now assume $\ext$ is conflict--free in $\ADF^{\SETAF}$, but not in $\SETAF$. Hence, there is an argument 
$e \in \ext$ s.t. $C_e(\ext \cap par(e)) = in$, but $\exists \ext' \subseteq \ext$ s.t. $\ext'Re$.
Again, by Translation \ref{trans:setafadf} it is easy to see that it cannot be the case. 
\end{proof}

\begin{restatable}{lemma}{lemmasarange}
\label{lemma:sarange}
Let $\ext$ be a conflict--free extension of $\SETAF$ (and thus of $\ADF^{\SETAF}$). 
The discarded set of $\ext$ in $\SETAF$ coincides with the discarded set of $\ext$ in $\ADF^{\SETAF}$.
\end{restatable} 

\begin{proof}
We will refer to the discarded set of $\ext$ in $\SETAF$ with $\ext^{att}$ in order to avoid confusion. 

Let $a \in A$ be an argument in $\ADF^{\SETAF}$. We can observe that any minimal decisively in interpretation
for $a$ will have an empty $\tvt$ part and the $\tvf$ one will correspond to those (minimal) subsets $T \subseteq A$
s.t. $\forall S \subseteq A$, if $S R a$ then $T \cap S \neq \emptyset$. 
We can thus construct trivial evaluations for $a$ that will always be acyclic.

Let $a \in \ext^{att}$ be in the discarded set of $\SETAF$. Therefore, $\exists \ext' \subseteq \ext$ s.t. $\ext' R a$. 
Based on the previous explanations, we can observe that for any minimal decisively in interpretation $v$ for $a$,
$v^\tvf \cap \ext' \neq \emptyset$. Hence, any evaluation constructed for $a$ will be blocked by $\ext$ in $\ADF^{\SETAF}$
and $\ext^{att} \subseteq \ext^+$.

By Lemma \ref{lemma:modrange}, the acceptance condition of any argument in $\ext^+$ in $\ADF^{\SETAF}$ 
evaluates to $out$ under $\ext$.. And by construction, the
acceptance condition of an argument is $out$ w.r.t. $\ext$ if $\exists \ext' \subseteq \ext$ attacking this argument in $\SETAF$.
Hence, whatever is in $\ext^+ \subseteq \ext^{att}$. We can therefore conclude that the discarded sets coincide.
\end{proof}

\begin{restatable}{lemma}{lemmasadefense}
\label{lemma:sadefense}
A conflict--free set of arguments $\ext$ defends an argument $a\in A$ in $\SETAF$ iff $a$ is decisively in w.r.t. $v_\ext$ in $\ADF^{\SETAF}$.
\end{restatable}

\begin{proof}
We will refer to the discarded set of $\ext$ in $\SETAF$ with $\ext^{att}$ in order to avoid confusion. 

Let $\ext \subseteq A$ be a conflict--free extension of $\SETAF$.
By Theorem \ref{thm:sacf}, $\ext$ is a conflict--free extension $\ADF^{\SETAF}$ as well. Moreover, by Lemma \ref{lemma:sarange},
 $\ext^{att}=\ext^+$.
Assume that $a$ is defended by $\ext$ in $\SETAF$, but is not decisively in w.r.t. $v_\ext$. 
If $a$ is not decisively in w.r.t. $v_\ext$, it means there exists a completion $v'$ of $v_\ext$ to $\ext \cup par(a)$ s.t. $C_a(v'^{\tvt} \cap par(a)) = out$. 
This means that $v'^\tvt \cap par(a)$ contains a set of arguments $\ext'$ s.t. $\ext'Ra$. 
Since the set can be mapped to $\tvt$ in the completion, none of its members
is mapped to $\tvf$ in $v_\ext$ and thus none of them appears in $\ext^+$. Consequently, none of them is in $\ext^{att}$ either. Therefore, $\ext$
could not have defended $a$ in $\SETAF$. We reach a contradiction.

Let $\ext \subseteq A$ be conflict--free in $\ADF^{\SETAF}$ and thus in $\SETAF$.
Assume that $a \in A$ is decisively in w.r.t. $v_\ext$, but is not defended by $\ext$. 
This means there exists a set of arguments $B$ s.t. $B R a$ and $B\cap \ext^{att} = \emptyset$. Consequently, there exists a set of arguments $B$
s.t. $C_a(B) = out$ and $B\cap \ext^+ = \emptyset$. If this is the case, then obviously $a$ cannot be decisively in and we reach a contradiction.
\end{proof}

\thmsetafadf*

\begin{proof}
Since SETAFs properly generalize AFs, it suffices to focus on them.

Due to the fact that the semantics classification collapses for $\ADF^{\SETAF}$ (see Theorems \ref{thm:extcollapse} and \ref{thm:setafadfnf}),
it suffices to focus on only conflict--free, grounded, model, and the cc--types of the ADF semantics. 

Conflict--freeness was already proved in Theorem \ref{thm:sacf}. The fact that admissible and cc--admissible extensions coincide follows straightforwardly from
Theorem \ref{thm:sacf} and Lemma \ref{lemma:sadefense}. Therefore, the preferred extensions coincide as well.
Due to the correspondence between decisiveness and defense as seen in Lemma \ref{lemma:sadefense}, 
complete and cc--complete extensions in both frameworks are also the same. 
By Theorem \ref{thm:compsetaf}, the grounded extension of $\SETAF$ is the least
w.r.t. set inclusion complete one. The grounded extension of $\ADF^{\SETAF}$ is the least w.r.t. set inclusion cc--complete one. 
Therefore, the grounded extension is the same for both frameworks.

Let us finish with the analysis of stability. Assume $\ext$ is stable in $\SETAF$, but not in $\ADF^{\SETAF}$. 
This means that $\ext$ is conflict--free in 
$\SETAF$ and 
$\ext^{att} = A\setminus E$. By Theorems \ref{thm:sacf}, \ref{thm:extcollapse}, \ref{thm:setafadfnf}
 and Lemma \ref{lemma:sarange}, $\ext$ is pd--acyclic conflict--free in $\ADF^{\SETAF}$ and $\ext^{att} = \ext^+$. 
Hence, $\ext^+ = A\setminus \ext$. All arguments in $\ext^+$ are decisively out w.r.t. $v_\ext$,
and thus there may be no argument $e \in \ext^+$ s.t. $C_e(\ext \cap par(e)) = in$. 
Therefore, the model and stable requirements in $\ADF^{\SETAF}$ are satisfied.

Every ADF stable extension is a model, which is conflict--free in $\ADF^{\SETAF}$ and thus also in $\SETAF$.
By Lemma \ref{lemma:modrange} and Theorems \ref{thm:extcollapse}, we have that $\ext^+ = A\setminus \ext$ in $\ADF^{\SETAF}$. Thus,
by Theorem \ref{lemma:sarange}, every argument in $A\setminus \ext$ is attacked by $\ext$. Hence, SETAF stability conditions are satisfied.
\end{proof}

\subsubsection{Translations for EAFCs}

In this section we will prove Theorems \ref{thm:eafcadfnf} and \ref{thmeafcadf}. As partial results for the latter,
we also introduce Theorem \ref{thm:eafccf} and Lemmas \ref{lemma:eafcrange} and \ref{lemma:eafcdefense}.

\thmeafcadfnf*

\begin{proof}
Let $a,b \in A$ be arguments s.t. $(a,b) \in R$. By strong consistency it means there is no other attack on $b$ that would
be defense attacked by a set containing $a$. If a given set $\ext$ has a subset defeating$_\ext$ $b$, then so does $\ext \cup \{a\}$. Therefore,
there is no subset $F$ of parents of $b$ in $\ADF^{\EAFC}$ s.t. $C_b(F) = out$ and $C_b (F \cup \{a\}) = in$. The $(a,b)$
link in $\ADF^{\EAFC}$ is thus an attacking one based on Definition \ref{def:badfaadf}. 
Furthermore, it cannot be supporting -- 
due to consistency, $C_b(\emptyset) = in$ and $C_b (\{a\}) = out$. 

Let now $a, b \in A$ be arguments s.t. there is $c \in A$, $G \subseteq A$, $a \in G$ and $(G, (c,b)) \in D$. Due to consistency, it cannot be the case
that $(a,b) \in R$. This means that if $\ext$ does not defeat$_\ext$ $b$, then neither does $\ext \cup \{a\}$. 
Therefore,
there is no subset $F$ of parents of $b$ in $\ADF^{\EAFC}$ s.t. $C_b(F) = in$ and $C_b (F \cup \{a\}) = out$. The $(a,b)$
link in $\ADF^{\EAFC}$ is thus a supporting one based on Definition \ref{def:badfaadf}. 
Therefore, $\ADF^{\EAFC}$ is a BADF.  

Let us now assume that $\EAFC$ is bounded hierarchical and let $( ((A_1, R_1), D_1)$, $...,$ $((A_n, R_n), D_n))$ be its partition satisfying
the requirements in Definition \ref{def:bheafc}.
Let us start with $((A_n, R_n), D_n)$. We can observe that as $D_n = \emptyset$,
then all of the parents of $a \in A_n$ are in $A_n$. Furthermore, they are only connected by the $R_n$ relation, which means
that all arguments in $A_n$ in $\ADF^{\EAFC}$ have Dung--style acceptance conditions. Therefore, 
every argument in $A_n$ has precisely one minimal decisively in interpretation that does not contain any $\tvt$ mappings. 
Hence, every argument in $A_n$ satisfies $a_0$ requirements of a pd--acyclic evaluation. This means that every partially acyclic
evaluation on $A_n$ will be indeed acyclic. Let us now focus on $((A_{n-1}, R_{n-1}), D_{n-1})$. Notice that
$D_{n-1} \subseteq A_n$.
Every argument $a \in A_{n-1}$ depends only on arguments in $A_{n-1} \cup A_n$. Furthermore, if a minimal decisively in interpretation
for $a$ contains $\tvt$ mappings, then those mappings can be in $A_n$ only. Therefore, any ordering on $A_n$ extended with any ordering
on $A_{n-1}$ will give us a pd--sequence of a pd--acyclic evaluation, independently of the chosen minimal decisively in interpretations for the arguments.
Therefore, the evaluations on $A_{n-1} \cup A_n$ will be acyclic. We can continue this line of reasoning
until we reach $((A_1, R_1), D_1)$ and the conclusion that every evaluation on $A = \bigcup_{i=1}^n A_n$ will be acyclic. 
Thus, $\ADF^{\EAFC}$ is an AADF$^+$. 
\end{proof}

\begin{restatable}{theorem}{thmeafccf}
\label{thm:eafccf}
A set of arguments $\ext$ is a conflict--free extension of $\EAFC$ iff it is a conflict--free extension of $\ADF^{\EAFC}$.
\end{restatable}

\begin{proof}
Let $\ext \subseteq A$ be a conflict--free extension of $\EAFC$. This means that given an argument $a \in \ext$,
it is either not attacked at all in $\ext$ or every attack carried out by a member of $\ext$ is defense attacked by a subset of $\ext$. Thus, from the functional
version of the acceptance conditions in Translation \ref{trans:eafcadfcons} we can observe that $C_a (\ext \cap par(a)) = in$. Consequently,
if $\ext$ is conflict--free in $\EAFC$, then every argument in $\ext$ has a satisfied acceptance condition w.r.t. $\ext$ in $\ADF^{\EAFC}$.
This means that $\ext$ is conflict--free in $\ADF^{\EAFC}$.

Let now $\ext \subseteq A$ be a conflict--free extension of $\ADF^{\EAFC}$. This means that for any argument $a \in \ext$,
$C_a(\ext \cap par(a)) = in$. By the construction of the condition it means that either there is no argument $b \in \ext$
s.t. $(b,a) \in R$, or for any such attack there is a subset of $\ext$ defense attacking it. Consequently, there are no defeats$_\ext$
in $\ext$ in $\EAFC$ and thus $\ext$ is conflict--free in $\EAFC$ as well.
\end{proof}

\begin{restatable}{lemma}{lemmaeafcrange}
\label{lemma:eafcrange}
Let $\ext$ be a conflict--free extension of $\EAFC$ (and thus of $\ADF^{\EAFC}$). 
The discarded set of $\ext$ in $\EAFC$ coincides with the partially acyclic discarded set of $\ext$ in $\ADF^{\EAFC}$.
\end{restatable} 

\begin{proof}
Let us first note on how (minimal) decisively in interpretations for arguments in $A$ look like. Due to the fact that
we are dealing with a strongly consistent framework, then from the propositional acceptance conditions
we can observe that for any attack subformula of the condition, the interpretation has to either map the attacker to $\tvf$ or at least one
defense attacking sets to $\tvt$. Thus, even though technically speaking EAFCs are attack--based frameworks, the minimal interpretations
can contain $\tvt$ assignments, which was not the case in e.g. AFs or SETAFs. If the framework was not consistent, then we could obtain
new
minimal decisively in interpretations that would be contained in the described ones. For example, the condition of $b$ in the framework
$(\{a,b \}, \{(a,b)\}, \{(a,(a,b))\})$ would be equivalent to $\top$ and thus an empty translation would have been also possible,
despite the fact that the argument is attacked by $a$ and $\{b\}$ is not an admissible extension of $\EAFC$.

Let $\ext \subseteq A$ be a conflict--free extension of $\EAFC$. By Lemma \ref{thm:eafccf}, $\ext$ is conflict--free in $\ADF^{\EAFC}$. 
We define the set $\ext^+$ as the collection of those arguments
$b \in A$ s.t. an argument $a \in \ext$ defeats$_\ext$ $b$ and there is a reinstatement set for this defeat on $\ext$. Clearly,
by conflict--freeness of $\ext$, $\ext \cap \ext^+ = \emptyset$.
We will show that this set is
equal to the partially discarded set $\ext^{p+}$ in $\ADF^{\EAFC}$. 

Let $b \in \ext^+$ in $\EAFC$. Assume it does not qualify for $\ext^{p+}$ 
in $\ADF^{\EAFC}$;
this means that $b$ has a partially acyclic evaluation $(F,G,B)$ on $A$ s.t. $B \cap \ext = \emptyset$ and $F \subseteq \ext$.
Let $G = (a_0,...,a_n)$ be the pd--sequence of the evaluation. Due to the construction
of the sequence, the $\tvt$ part of the decisively in interpretation $v_{a_0}$ used for $a_0$ in the construction of $(F,G,B)$ 
is contained in $\ext$. 
Since
$B \cap \ext =\emptyset$, $v_{a_0}^\tvf \cap \ext = \emptyset$. Therefore, by the construction of the decisively in interpretations in consistent
frameworks and the nature of the
acceptance conditions in $\ADF^{\EAFC}$, this means means that there is no $x \in \ext$ s.t. $x$ defeats$_\ext$ $a_0$.  
Thus, $a_0$ could not have been present in $\ext^+$ in $\EAFC$.
Let us continue with $a_1$. Its minimal decisively in interpretation $v_{a_1}$ that has been used in construction of $(F,G,B)$
has a $\tvt$ part that is a subset of $\ext \cup \{a_0\}$. We can again observe that $v_{a_1}^\tvf \cap \ext = \emptyset$. From the construction
of interpretations and conditions, this means that if there is an attack carried out at $a_1$ by some element of 
$\ext$, then it is defense attacked by a subset of $\ext \cup \{a_0\}$. Since $a_0$ is not 
defeated by any argument in $\ext$, then either no argument in $\ext$ defeats $a_1$ (i.e. no attacker of $a_1$ is present or $a_0 \in \ext$)
or for no defeat by $\ext$ on $a_1$ there is a reinstatement set on $\ext$. Consequently, $a_1$ does not qualify for $\ext^+$ in $\EAFC$.
We can continue reasoning in this manner till we reach $a_n = b$ and the conclusion that if $b$ has a partially acyclic evaluation $(F,G,B)$
s.t. $F' \subseteq \ext$ and $B \cap \ext = \emptyset$ in $\ADF^{\EAFC}$, then it cannot be in $\ext^+$ in $\EAFC$.

We have just shown that $\ext^+ \subseteq \ext^{p+}$. We now need to prove that there is no argument $b \in \ext^{p+}$
in $\ADF^{\EAFC}$ that is not in $\ext^+$ in $\EAFC$. 
Assume it is not the case; therefore, either no argument in $\ext$ defeats$_\ext$ $b$ 
or no such defeat has a reinstatement set on $\ext$ in $\EAFC$, even though $b \in \ext^{p+}$ in $\ADF^{\EAFC}$. 
Let us focus on the first case. If there is no defeat, then there is either no
attack on $b$ from $\ext$ in the first place, or for every attack there is a subset of $\ext$ carrying out an appropriate defense attack.
Consequently, we can observe that the acceptance condition of $b$ w.r.t. $\ext \cap par(b)$ should be mapped to $in$. Thus, by Proposition 
\ref{prop:range},
$b$ could not have been in $\ext^{p+}$ and we reach a contradiction with the assumptions. 
Let us now focus on the case where there is a defeat on $b$ by an argument $d \in\ext$, but it lacks a reinstatement set on $\ext$. 
By Theorem \ref{thm:cdefrei}, there exists a sequence of distinct defense attacks $((Z_1, (x_1, y_1)), ..., (Z_n, (x_n, y_n)))$ 
s.t. $(x_n, y_n) = (d,b)$, each  $(x_i, y_i)$ attack is unique, no argument in $\ext$ defeats$_\ext$ any element $z \in Z_1$, and 
for every other $(Z_i, (x_i, y_i))$ in the sequence, either no argument $h \in \ext$ defeats$_\ext$ any element $z' \in Z_i$ or
for every such defeat there is a set of arguments $L \subseteq A$ s.t. $(L,(h,z') \in \{(Z_1, (x_1, y_1)),...,(Z_{i-1}, (x_{i-1}, y_{i-1}))\}$.
Let us start with the set $Z_1$. We can observe that if $\ext$ does not defeat$_\ext$ any argument in $Z_1$, 
then the conditions of the arguments in $Z_1$ are in fact satisfied by $\ext$. Thus, no element of $Z_1$ is in the partially 
acyclic discarded set by Proposition \ref{prop:range}. Let us now consider $Z_2$ and let $z \in Z_2$ be an argument. If it is not defeated$_\ext$ by $\ext$, 
then we come back to the previous case and can show that $z$ cannot be in the partially acyclic discarded set. If it is defeated$_\ext$,
then the condition of $z$ is out w.r.t. $\ext$. However, we can observe that by the construction, the condition of $z$ w.r.t. $\ext \cup Z_1$
is in, and as no element in $Z_1$ is in the partially acyclic discarded set, then the argument cannot be decisively out w.r.t. the partially acyclic range.
Thus, it is not in the partially acyclic discarded set by Proposition \ref{prop:range}. We can therefore show
that $Z_2 \cap \ext^{p+} = \emptyset$. We can continue this line of reasoning until we reach $Z_n$ and the result
that $Z_n \cap \ext^{p+} = \emptyset$. Consequently, $y_n$ cannot be decisively out w.r.t. the partially acyclic range either
and $y_n = b \notin \ext^{p+}$. We reach a contradiction with the assumptions. Therefore, $\ext^{p+} \subseteq \ext^{+}$.
We can thus finally conclude that $\ext^+ = \ext^{p+}$.
\end{proof}

\begin{restatable}{lemma}{lemmaeafcdefense}
\label{lemma:eafcdefense}
A conflict--free set of arguments $\ext$ defends an argument $a\in A$ in $\EAFC$ iff $a$ is decisively in w.r.t. the partially acyclic range 
$v_\ext^p$ of $\ext$ in $\ADF^{\EAFC}$.
\end{restatable}

\begin{proof}
In Theorem \ref{thm:eafccf} we have shown that the conflict--free extensions of $\EAFC$ and $\ADF^{\EAFC}$ coincide. 
In Lemma \ref{lemma:eafcrange}, we have proved that the set of arguments defeated by $\ext$ with a reinstatement set on $\ext$ in $\EAFC$ 
equals the partially acyclic discarded set of $\ext$ in $\ADF^{\EAFC}$. Now, we need to prove that an argument $a \in A$ is defended
by $\ext$ in $\EAFC$ iff it is decisively in w.r.t. the partially acyclic range interpretation $v_\ext^p$ of $\ext$ in $\ADF^{\EAFC}$.

Let us start with left to right direction. If an argument $a$ is defended by $\ext$, then every argument $b \in \ext$ s.t.
$b$ defeats$_\ext$ $a$, is in turn defeated with reinstatement by $\ext$. Therefore, $a$ is defended iff every argument $b\in A$
defeating it is in $\ext^+$. Let us now consider an argument $c$ s.t. $(c,a) \in R$, but $c$ does not defeat$_\ext$ $a$. This means
that there is a suitable defense attack carried out by a set $F \subseteq \ext$. We can now shift to $\ADF^{\EAFC}$.
Every attacker of $a$, be it $b$ style (i.e. it becomes a defeater) or $c$ style (i.e. does not become a defeater), 
has a corresponding $att$ formula in the condition of $a$ and this formula
is not equivalent to $\top$ due to the strong consistency of $\EAFC$. If it is a formula $att^b_a$, then
we can observe that as $b$ is mapped to $\tvf$ by the partially acyclic discarded range, the formula evaluates to true under this range
and will remain such independently of what is assigned to the remaining arguments in the formula. If it is a formula $att^c_a$, then
the disjunction of conjunctions corresponding to the defense attackers evaluates to true and thus the whole $att^c_a$ is true.
Moreover, it will stay such, no matter what new arguments come into play. Consequently, the condition of $a$ is $in$
under the partially acyclic range and will remain $in$ for any of its completions to $A$. Thus, $a$ is decisively in w.r.t. the partially acyclic 
range of $\ext$.

Let us continue with the right to left direction. If an argument $a$ is decisively in w.r.t. the partially acyclic range,
then its condition is $in$ w.r.t. every completion of the range to $A$. 
This means that every $att^b_a = \neg b \lor (\bigwedge B_1 \lor ... \bigwedge B_m)$ subformula of the
acceptance condition evaluates to true under the acyclic range and remains such under every completion. Therefore,
it is either $b$ that has to be assigned $\tvf$ by the range or at least one set $B_i$ has all arguments assigned $\tvt$ by the range.
If it is the first case, then by Lemma \ref{lemma:eafcrange}, $b \in \ext^+$ and if the attack from $b$ is a defeat, 
then $a$ is defended from $b$ by $\ext$ in $\EAFC$. If it is the latter,
then we can observe that the attack from $b$ on $a$ does not become a defeat. Since the $att$ subformulas account
for all attackers of $a$, we can conclude that $\ext$ defends $a$.
\end{proof}

\thmeafcadf*

\begin{proof} 
With the help of Theorem \ref{thm:eafccf}, Lemmas \ref{lemma:eafcrange} and \ref{lemma:eafcdefense},
it can be shown that $\ext \subseteq A$ is a $\sigma$--extension of $\EAFC$, where $\sigma \in \{$ admissible, complete
preferred $\}$ iff it is a $ca_2-\sigma$--extension of $\ADF^{\EAFC}$. What remains to be proved is the relation between
stable extensions and models, and the grounded and acyclic grounded extensions.

Let $\ext \subseteq A$ be a stable extension of $\EAFC$. This means it is conflict--free and defeats$_\ext$ every argument $a \in A\setminus \ext$.
We can observe that every defeat$_\ext$ originating from $\ext$ will be a trivial reinstatement set for any of these defeats$_\ext$.
Therefore, from Theorem \ref{thm:eafccf} and Lemma \ref{lemma:eafcrange}, 
it follows that $\ext$ is conflict--free in $\ADF^{\EAFC}$ and that every argument $a \in A \setminus \ext$
is in the partially acyclic discarded set. By Proposition \ref{prop:range} it holds that for every such $a$, $C_a( \ext \cap par(a)) = out$.
Therefore, $\ext$ is a model of $\ADF^{\EAFC}$. As observed in Example \ref{ex:adf}, $\ext$ does not need to be stable in $\ADF^{\EAFC}$.

Let $\ext \subseteq A$ be a model of $\ADF^{\EAFC}$. By Theorem \ref{thm:eafccf}, it is conflict--free in $\EAFC$. 
By Lemma \ref{lemma:modrange}, from
the fact that $\ext$ is a model it follows that
every argument in $A \setminus \ext$ is in the partially acyclic discarded set. Consequently, it is also in $\ext^+$ in $\EAFC$, and by the
definition of this set is defeated$_\ext$ by $\ext$. Therefore, $\ext$ is stable in $\EAFC$.

In order to show that the grounded extension in $\EAFC$ and the acyclic grounded in $\ADF^{\EAFC}$ correspond, we can
use the iterating from the empty set approach \cite{ModgilP10,report:semanticsrev}.
Let us start with $\ext = \ext' = \emptyset$. The set $\ext$ is conflict--free in $\EAFC$ and $\ext'$ is pd--acyclic conflict--free in $\ADF^{\EAFC}$. 
They are also (aa--)admissible in their respective frameworks. 
Since $\ext'$ is pd--acyclic conflict--free, then the partially acyclic range of $\ext'$ is in fact acyclic by Lemma \ref{lemma:disc}.
Therefore, if we perform an iteration and add to $\ext$ the arguments it defends in $\EAFC$ and to $\ext'$ those that are decisively in
w.r.t. the acyclic range of $\ext'$ in $\ADF^{\EAFC}$, then it is still the case that $\ext = \ext'$. Moreover, by Lemma \ref{fund1},
$\ext'$ is still aa--admissible and thus pd--acyclic conflict--free. From the admissibility of $\ext'$ follows the admissibility of $\ext$.
We can now repeat the iteration and again observe that $\ext = \ext'$. We can continue in this manner until there are no arguments left
and observe that $\ext = \ext'$ and $\ext$ is grounded in $\EAFC$
while $\ext'$ acyclic grounded in $\ADF^{\EAFC}$.
\end{proof}

\subsubsection{Translations for AFNs}

In this section we will include the proofs concerning the translation from AFNs to ADFs. Theorems \ref{thm:consafnadfnf},
\ref{thm:afnadfsem} and Lemma \ref{lemma:ppd} were mentioned in the text. We use Lemmas \ref{lemma:chcf}, \ref{lemma:acydattafn}
and \ref{thm:afnadfdef} as partial results leading to Theorem \ref{thm:afnadfsem}.

\consafnadfnf*

\begin{proof} 
Let us assume that $\ADF^{\AFN}$ is not a BADF. This means there exists a link $(a,b) \in L$ in $\ADF^{\AFN}$ that is neither
supporting nor attacking. Consequently, there exists $\ext \subseteq par(b)$ s.t. $C_b(\ext) = in$ and $C_b(\ext \cup \{a\}) = out$
and a set $\ext' \subseteq par(b)$ s.t. $C_b(\ext') = out$ and $C_b(\ext' \cup \{a\}) = in$. 
Based on Translation \ref{trans:afn}, we can observe that if $C_b(\ext) = in$, then $\ext \cap F \neq \emptyset$ for 
every set $F \subseteq A$ s.t. $F N b$ and there is no argument $e \in \ext$ s.t. $e R b$. Thus, if $C_b(\ext \cup \{a\}) = out$,
then it can only be the case that $a R b$. Therefore, there cannot exist a set of arguments $\ext'$ s.t. $C_b(\ext' \cup \{a\}) = in$,
as by definition in every such case $C_b(\ext' \cup \{a\}) = out$.
Hence, $\ADF^{\AFN}$ is a BADF.
\end{proof}

\ppd*

\begin{proof}
Let $\ext \subseteq A$ be a set of arguments, $e \in \ext$ and $(a_0,...,a_n)$ a powerful sequence for $e$.
We will show it satisfies the pd--sequence requirements. 

First of all, the $a_n=e$ condition is satisfied. Secondly, we have that for $a_0$ there is no $B\subseteq A$ s.t. $BNa_0$. This means that
$a_0$ faces only binary attack and its condition basically consists only of the $att$ part. We can show that
$a_0$ has a single minimal decisively in interpretation that maps every attacker of $a_0$ to $\tvf$. 
The $\tvt$ part is empty
and thus the interpretation satisfies the pd--evaluation criterion of $a_0$. 

Finally, in the powerful sequence, for every nonzero $a_i$ it holds that for each
$B\subseteq A$ s.t. $BNa_i$, $B \cap \{a_0,...,a_{i-1}\}\neq \emptyset$. 
Let $\ext_i = \{a_0,...,a_{i-1}\} \cap par(a_i)$.
Since $\AFN$ is strongly consistent, no argument in $\ext$ is an attacker of $a_i$. 
Thus, by the construction
of $\ADF^{\AFN}$ it holds that $C_{a_i}(\ext_i) = in$. 
An interpretation assigning $\tvt$ to $\ext_i$ and $\tvf$ to $A\setminus \ext_i$ will be a decisively in interpretation for $a_{i}$. 
Thus, we can extract a minimal interpretation $v$ from it, which will assign $\tvt$ to a subset $\ext'_i \subseteq \ext_i$
and $\tvf$ to all those arguments $b \in A$ s.t. $b R a_{i}$. Based on this, we can conclude that $v$ satisfies the pd--sequence condition. 
Therefore, we obtain an acyclic pd--evaluation $((a_0,...,a_n), \bigcup_0^n \{a_i\}^-)$ for $e$ on $\ext$ corresponding
to the powerful sequence $(a_0,...,a_n)$.

%
Let $\ext \subseteq A$ be a set of arguments, $e \in \ext$ and $((a_0,...,a_n),B)$ an acyclic pd--evaluation for $e$. We will show that the sequence
part satisfies the powerful conditions. Again, the $a_n = e$ condition is easily met. The decisively in interpretation for $a_0$ consists only from negative
mappings, which by Translation \ref{trans:afn} come from the attackers of $a_0$. As $a_0$ is strongly consistent, none of those attackers is also a supporter,
and thus we can conclude that there exists no supporting set for $a_0$. Another powerful requirement is met. Now, we know that for 
every nonzero $a_i$ and its minimal decisively in interpretation $v_i$, $v_i^\tvt \subseteq \{a_0,...,a_{i-1}\}$. By construction of the arguments we know
that $\forall Z\subseteq A$ s.t. $Z N a_i$, $v_i^\tvt \cap Z \neq \emptyset$. Consequently, $Z \cap \{a_0,...,a_{i-1}\} \neq \emptyset$ 
and the final powerful requirement
is satisfied. Therefore, the pd--sequence of the evaluation produces a powerful sequence.
\end{proof}

\begin{restatable}{lemma}{chcf}
\label{lemma:chcf}
Let $\AFN = (A, R, N)$ be a strongly consistent AFN, $\ADF^{\AFN} =(A,L,C)$ its corresponding ADF.
A set of arguments $\ext \subseteq A$ is strongly coherent in $\AFN$ iff it is a pd--acyclic conflict--free extension of $\ADF^{\AFN}$.
\end{restatable}

\begin{proof} 
Let us assume that $\ext$ is strongly coherent in $\AFN$, but not pd--acyclic conflict--free in $\ADF^{\AFN}$. By Lemma \ref{lemma:ppd}
we know that every argument in $\ext$ possesses a pd--acyclic evaluation on $\ext$. What remains to be shown is that every argument
has an evaluation on $\ext$ that is also unblocked.
By Lemma \ref{lemma:ppd} we can create an evaluation corresponding to the powerful sequence of $e$ on $\ext$.
The blocking set of such an evaluation corresponds exactly to the union of attackers of all its sequence members. 
As all the members of the pd--sequence of this evaluation
are in $\ext$, it has to be the case that an element of the blocking set is accepted. 
However, it would clearly breach the conflict--freeness of $\ext$ in $\AFN$ and we reach
a contradiction. Therefore, $\ext$ is pd--acyclic conflict--free in $\ADF^{\AFN}$.

Let us now assume that $\ext$ is pd--acyclic conflict--free in $\ADF^{\AFN}$, but not strongly coherent in $\AFN$. 
By Lemma \ref{lemma:ppd}, $\ext$ is at least coherent. 
If $\ext$ is not conflict--free in $\AFN$, it means that $\exists x,y \in \ext$
s.t. $xRy$. However, by strong consistence of $\AFN$ and Translation \ref{trans:afn}, 
it would mean that $C_y(\ext\cap par(y)) = out$. Consequently, $\ext$ could not have been conflict--free in $\ADF^{\AFN}$,
and as every pd--acyclic conflict--free extension is also just conflict--free, we reach a contradiction. 
Hence, if $\ext$ is pd--acyclic conflict--free in $\ADF^{\AFN}$, then it is strongly coherent in $\AFN$.
\end{proof}

\begin{restatable}{lemma}{acydattafn}
\label{lemma:acydattafn}
Let $\AFN = (A, R, N)$ be a strongly consistent AFN, $\ADF^{\AFN} =(A,L,C)$ its corresponding ADF.
Let $\ext \subseteq A$ be strongly coherent in $\AFN$ and thus pd--acyclic conflict--free in $\ADF^{\AFN}$.  
Then $\ext^{att}$ coincides with the acyclic discarded set of $\ext$. 
\end{restatable}

\begin{proof}
If every coherent set containing $a$ is attacked by $\ext$, it means that every powerful sequence for $a$ is attacked by $\ext$.
By Lemma \ref{lemma:ppd}, we have that every powerful sequence corresponds to an acyclic pd--evaluation. 
As seen in the proof, attackers of the members of this sequence
form the blocking set of the evaluation. 
Thus, if $\ext$ attacks  a member of the powerful sequence, it means that an argument from the blocking set of the evaluation
is in $\ext$. Therefore, the evaluation is blocked, and whatever is in $\ext^{att}$ is in $\ext^{a+}$.

Now let us assume there is an argument $a \in \ext^{a+}$, but not in $\ext^{att}$. 
This means that $a$ has an unattacked powerful sequence, but every of its pd--acyclic evaluations $(F,B)$ is blocked through the blocking set.
By Lemma \ref{lemma:ppd} we can construct a pd--evaluation corresponding to the unattacked sequence. Since
the blocking set of the evaluation is composed of the attackers of members of the powerful sequence, it cannot be the case
that there is no $b \in \ext$ attacking the sequence and at the same time $\ext \cap B \neq \emptyset$. We reach a contradiction.
Therefore, whatever is in $\ext^{a+}$ is also in $\ext^{att}$. 
\end{proof}

\begin{restatable}{theorem}{afnadfdef}
\label{thm:afnadfdef}
Let $\AFN = (A, R, N)$ be a strongly consistent AFN and $\ADF^{\AFN} =(A,L,C)$ its corresponding ADF. 
Let $\ext \subseteq A$ be strongly coherent in $\AFN$ and thus pd--acyclic conflict--free in $\ADF^{\AFN}$. Then
$\ext$ defends an argument $a\in A$ in $\AFN$ iff this argument is decisively in w.r.t $v_\ext^a$ in $\ADF^{\AFN}$.
\end{restatable}

\begin{proof}
We will use the formulation of defense in AFNs from Lemma \ref{lemma:afndefatt}.

Let us  assume that $a$ is defended in $\AFN$,
but is not decisively in w.r.t. $v_\ext^a$. This means there exists at least one completion $v'$ of the acyclic range interpretation that outs the acceptance
condition of $a$. Let $\ext' = v'^{\tvt}$. According to Translation \ref{trans:afn}, the condition of $a$ is not satisfied iff there exists $b \in \ext'$ s.t.
$bRa$ or there exists $C \subseteq A$ s.t. $C N a$ and $C \cap \ext' = \emptyset$. If it is the first case, then from the fact that
$\ext^{att} = \ext^{a+}$ by Lemma \ref{lemma:acydattafn}, it follows that
there is an attacker $b$ of $a$ not included in $\ext^{att}$. Thus, $a$ could have not been defended in $\AFN$. If it is the latter case,
it means that there exists $C \subseteq A$ s.t. $C N a \land C \cap \ext = \emptyset$ as well. Consequently, $\ext \cup \{a\}$ could not have been coherent.
We reach a contradiction. Therefore, if an argument $a$ is defended by $\ext$ in $\AFN$, then it is decisively in w.r.t $v_\ext^a$ in $\ADF^{\AFN}$.

Let us now assume that $a$ is decisively in w.r.t. $v_\ext^a$, but is not defended in $\AFN$. This means that either there is an argument 
$b \in A$ s.t. $b R a$ and $b \notin \ext^{att}$, or $\ext \cup \{a\}$ is not coherent. By Translation \ref{trans:afn} and Lemma \ref{lemma:acydattafn}, 
it is easy to see that if it were the first case, then
$a$ could not have been decisively in w.r.t. the acyclic range of $\ext$. 
Let us thus assume that the issue lies in the coherence. Since we know that $\ext$ is strongly coherent, $a$
is the only argument that would not have a powerful sequence on $\ext \cup \{a\}$. This means that either there is no powerful sequence for $a$
to start with, or there is a set $C\subseteq A$ s.t. $CNa$ and $C\cap \ext = \emptyset$. 
If it is the first case, then by Lemma \ref{lemma:ppd} there is no pd--acyclic evaluation
for $a$ in $\ADF$. Consequently, it has to mapped to false by $v_\ext^a$ and is therefore decisively out w.r.t. it by Proposition \ref{prop:range}. 
We reach a contradiction with the assumption it is decisively in. If it is the latter case,
then by the Translation \ref{trans:afn} the acceptance condition of $a$ could not have been satisfied by $\ext$. Hence, $a$ could not have been
decisively in w.r.t $v_\ext^a$ and we reach a contradiction. We can therefore conclude that if $a$ is decisively in w.r.t. $v_\ext^a$ in $\ADF^{\AFN}$, 
then it is defended by $\ext$ in $\AFN$.
\end{proof}

\afnadfsem*

\begin{proof}
Let $\ext$ be an admissible extension in $\AFN$. By Lemma \ref{lemma:chcf} and Theorem \ref{thm:afnadfdef} we know that it is pd--acyclic conflict--free
in $\ADF^{\AFN}$ and that all arguments in $\ext$ are decisively in w.r.t $v_\ext^a$. Since the members of the blocking sets correspond to the attackers
of the arguments, they are naturally falsified in the range interpretation. Consequently, all aa--admissible criterions are satisfied. 
The other way around follows straightforwardly from the theorems.

We now know that the admissible extensions of $\AFN$ and $\ADF^{\AFN}$ coincide. Thus, the maximal w.r.t. set inclusion admissible sets are the same,
and $\ext$ is preferred in $\AFN$ iff it is aa--preferred in $\ADF^{\AFN}$.

The completeness follows straightforwardly from admissibility and Theorem \ref{thm:afnadfdef}. We can use Theorem \ref{thm:compafn}
 in order to show that $\ext$ is grounded in $\AFN$ iff it is acyclic grounded in $\ADF^{\AFN}$.

What remains to be shown is the correspondence of stable semantics. Let $\ext$ be AFN stable. By Lemma \ref{lemma:chcf} we know that $\ext$ is then
at least pd--acyclic conflict--free in $\ADF^{\AFN}$. It is easy to see by the definition of the deactivated set and Translation \ref{trans:afn}, that the acceptance
condition of every
argument $a \notin \ext$ will be out. Thus, $\ext$ satisfies the model criterion and we can conclude that it is ADF stable.

Let now $\ext$ be ADF stable.Since $\ext$ is also a model, then we know by Lemma \ref{lemma:modrange} that $\ext^{a+} = A\setminus \ext$.
We know it is pd--acyclic conflict--free, thus at least strongly coherent in $\AFN$ by Lemma\ref{lemma:chcf}. By this and Lemma \ref{lemma:acydattafn} 
we can conclude that $\ext^{a+}$ coincides with $\ext^{att}$. Thus, by Lemma \ref{lemma:afnstb2} $\ext$ is AFN stable. 
\end{proof} 

\end{document}